\documentclass{article} % For LaTeX2e
\usepackage{iclr2024_conference,times}

\usepackage{amsmath,amsfonts,bm}

\def\eqref#1{equation~\ref{#1}}
\def\1{\bm{1}}

\def\vmu{{\bm{\mu}}}

\def\vb{{\bm{b}}}
\def\vc{{\bm{c}}}

\def\vf{{\bm{f}}}

\def\vr{{\bm{r}}}

\def\vx{{\bm{x}}}

\def\mI{{\bm{I}}}

\DeclareMathAlphabet{\mathsfit}{\encodingdefault}{\sfdefault}{m}{sl}
\SetMathAlphabet{\mathsfit}{bold}{\encodingdefault}{\sfdefault}{bx}{n}

\newcommand{\E}{\mathbb{E}}

\newcommand{\KL}{D_{\mathrm{KL}}}

\usepackage{graphicx}
\usepackage{caption}
\usepackage{subfig}
\usepackage{hyperref}
\hypersetup{
	colorlinks=true,
	linkcolor=red,
	filecolor=blue,      
	urlcolor=magenta,
	citecolor=black,
}
\usepackage{url}
\usepackage{xspace}
\usepackage{booktabs}
\usepackage{multirow}
\newcommand{\method}{\textsc{ContextDiff}\xspace}
\usepackage[capitalize,noabbrev]{cleveref}
\crefname{section}{Sec.}{Secs.}
\Crefname{section}{Section}{Sections}
\Crefname{table}{Table}{Tables}
\crefname{table}{Tab.}{Tabs.}
\newtheorem{lemma}{lemma}
\newtheorem{proposition}{Proposition}
\newtheorem{proof}{Proof}

\title{Contextualized Diffusion Models for Text-Guided Image and Video Generation}

\author{Ling Yang$^{1*\dag}$\quad Zhilong Zhang$^{1*}$\quad Zhaochen Yu$^{1}$\thanks{Contributed equally.}\quad Jingwei Liu$^{1}$\quad Minkai Xu$^{2}$\\ \textbf{Stefano Ermon}$^{2}$\quad \textbf{Bin Cui}$^{1}$\thanks{Corresponding authors.} \\
$^{1}$Peking University\quad $^{2}$Stanford University\\
\texttt{\{yangling0818, bityzcedu, jingweiliu1996\}@163.com}\\
\texttt{\{minkai, ermon\}@cs.stanford.edu,\{zzl2018math, bin.cui\}@pku.edu.cn}
}

\iclrfinalcopy % Uncomment for camera-ready version, but NOT for submission.
\begin{document}

\maketitle

\begin{abstract}
Conditional diffusion models have exhibited superior performance in high-fidelity text-guided visual generation and editing. Nevertheless, prevailing text-guided visual diffusion models primarily focus on incorporating text-visual relationships exclusively into the reverse process, often disregarding their relevance in the forward process. This inconsistency between forward and reverse processes may
limit the precise conveyance of textual semantics in visual synthesis results. To address this issue, we propose a novel and general contextualized diffusion model (\method) by incorporating the cross-modal context encompassing interactions and alignments between text condition and visual sample into forward and reverse processes.
We propagate this context to all timesteps in the two processes to adapt their trajectories, thereby facilitating cross-modal conditional modeling. 
We generalize our contextualized diffusion to both DDPMs and DDIMs with theoretical derivations, and demonstrate the effectiveness of our model in evaluations with two challenging tasks: text-to-image generation, and text-to-video editing. 
In each task, our \method achieves new state-of-the-art performance, significantly enhancing the semantic alignment between text condition and generated samples, as evidenced by quantitative and qualitative evaluations. Our code is available at \href{https://github.com/YangLing0818/ContextDiff}{https://github.com/YangLing0818/ContextDiff}
% Notably, our \method can generalize to different conditional diffusion models with consistent improvement on both quantitative and qualitative results.  
\end{abstract}

\section{Introduction}
Diffusion models \citep{yang2023diffusion} have made remarkable progress in visual generation and editing. They are first introduced by \citet{sohl2015deep} and then improved by \citet{song2019generative} and \citet{ho2020denoising}, and can now generate samples with unprecedented quality and diversity \citep{rombach2022high,yang2023improving,podell2023sdxl,yang2024mastering,yang2024structure}. As a powerful representation space for multi-modal data, CLIP latent space \citep{radford2021learning} is widely used by diffusion models to semantically modify images/videos by moving in the direction of any encoded text condition for controllable text-guided visual synthesis \citep{ramesh2022hierarchical,saharia2022photorealistic,ho2022imagen,molad2023dreamix,wu2022tune,khachatryan2023text2video}.

Generally, text-guided visual diffusion models 
gradually disrupt visual input by adding noise through a fixed forward process, and learn its reverse process to 
generate samples from noise in a denoising way by incorporating clip text embedding. For example, text-to-image diffusion models usually estimate the similarity between text and noisy data to guide pretrained unconditional DDPMs \citep{dhariwal2021diffusion,nichol2022glide}, or directly train a
conditional DDPM from scratch by incorporating text
into the function approximator of the reverse process \citep{rombach2022high,ramesh2022hierarchical}. Text-to-video diffusion models mainly build upon pretrained DDPMs, and extend them with designed temporal modules (\textit{e.g.}, spatio-temporal attention) and DDIM \cite{song2020denoising} inversion for both temporal and structural consistency \citep{wu2022tune,qi2023fatezero}.

% The key idea of diffusion models is to disrupt images into noise through a fixed forward process and learn its reverse process to generate samples from noise in a denoising way. 
\begin{figure*}[h]
\centering
% \raggedleft
% \flushleft
\vspace{-0.17in}
\includegraphics[width=0.99\linewidth]{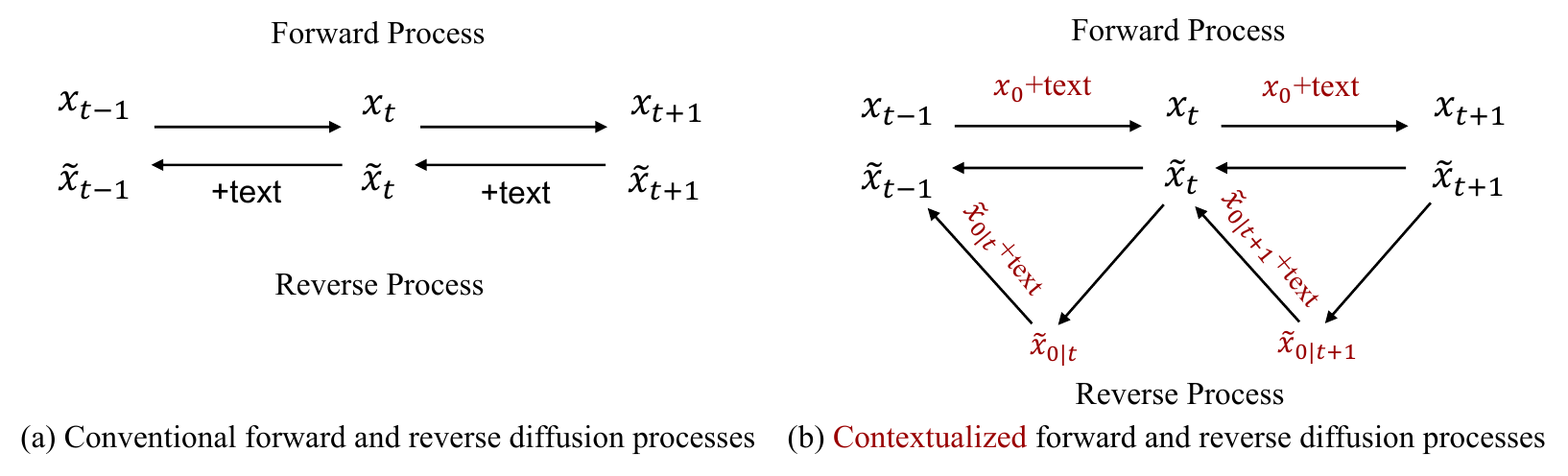}
\vspace{-0.07in}
\caption{
A simplified illustration of text-guided visual diffusion models with (a) conventional forward and reverse diffusion processes, (b) our contextualized forward and reverse diffusion processes.
$\Tilde{x}_0$ denotes the estimation of visual sample by the denoising network at each timestep.}
\vskip -0.2in
\label{fig:motivation}
\end{figure*}
Despite all this progress, there are common limitations in the majority of existing text-guided visual diffusion models. They typically employ an unconditional forward process but rely on a text-conditional reverse process for denoising and sample generation.
This inconsistency in the utilization of text condition between forward and reverse processes would constrain the potential of  conditional diffusion models. Furthermore, they usually neglect the cross-modal context, which encompasses the interaction and alignment between textual and visual modalities in the diffusion process, which may limit the precise expression of textual semantics in visual synthesis results.

To address these limitations, we propose a novel and general cross-modal contextualized diffusion model (\method) that harnesses cross-modal context to facilitate the learning capacity of cross-modal diffusion models.
% \newpage
% In this paper, we propose a novel and general cross-modal contextualized diffusion model (\method) for text-guided visual synthesis.
As illustrated in \cref{fig:motivation}, we compare our contextualized diffusion models with conventional text-guided diffusion models. We incorporate the cross-modal interactions between text condition and image/video sample into the forward process, serving as a context-aware adapter to optimize diffusion trajectories.
Furthermore, to facilitate the conditional modeling in the reverse process and align it with the adapted forward process, we also use the context-aware adapter to adapt the sampling trajectories. 
In contrast to traditional textual guidance employed for visual sampling process \citep{rombach2022high,saharia2022photorealistic}, our \method offers a distinct approach by providing enhanced and contextually informed guidance for visual sampling.
We generalize our contextualized diffusion to both DDPMs and DDIMs for benefiting both cross-modal generation and editing tasks, and provide detailed theoretical derivations.
We demonstrate the effectiveness of our \method in two challenging text-guided visual synthesis tasks: text-to-video generation and text-to-video editing. Empirical results reveal that our contextualized diffusion models can consistently improve the semantic alignment between text conditions and synthesis results over existing diffusion models in both tasks.

To summarize, we have made the following contributions: \textbf{(i)} To the best of our knowledge, We for the first time propose \method to consider cross-modal interactions as context-aware trajectory adapter to contextualize both forward and sampling processes in text-guided visual diffusion models. \textbf{(ii)} We generalize our contextualized diffusion to DDPMs and DDIMs with thereotical derivations for benefiting both cross-modal visual generation and editing tasks. \textbf{(iii)} Our \method achieves \textbf{new state-of-the-art performance on text-to-image generation and text-to-video editing tasks (in \cref{fig-video-qualitative})}, consistently demonstrating the superiority of our \method over existing diffusion models with both quantitative and qualitative comparisons.

\begin{figure*}[ht]
\centering
% \vskip -0.2in
\includegraphics[width=1\textwidth]{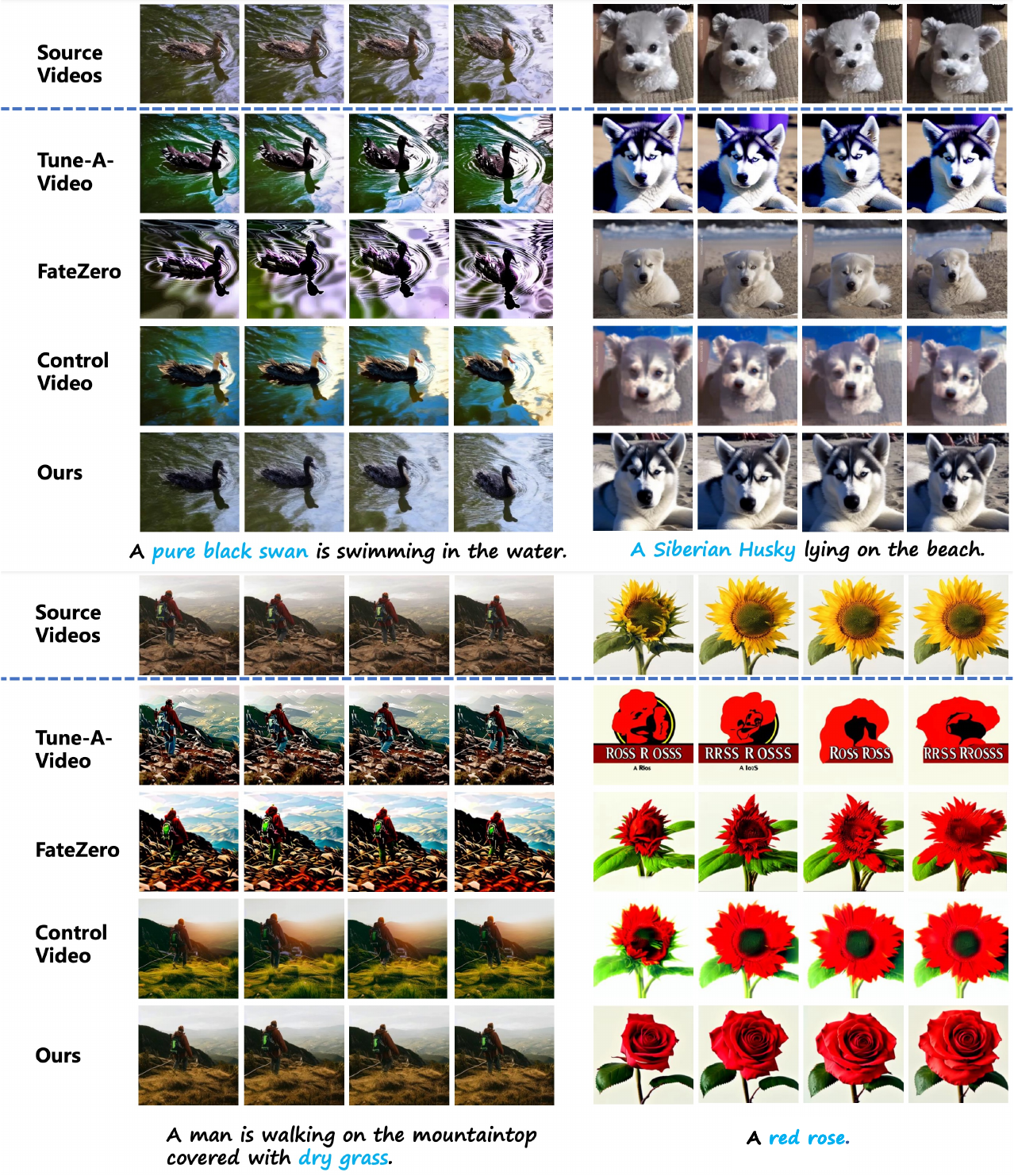}
% \vskip -0.1in
\caption{\textbf{Qualitative comparison in text-to-video editing}, edited text prompt is denoted in color. Our \method achieves best semantic alignment, image fidelity, and editing quality.}
% \vskip -0.2in
\label{fig-video-qualitative}
\end{figure*}

\section{Related Work}
\paragraph{Text-Guided Visual Diffusion Models}
% Diffusion models have make remarkable progress in text-guided visual synthesis, such as text-to-image generation and text-to-video editing. 
Text-to-image diffusion models \citep{yang2023improving,podell2023sdxl,yang2024mastering,zhang2024realcompo} mainly incorporate the text semantics that extracted by pretrained language models into the image sampling process \citep{nichol2022glide}. For example, 
% GLIDE \citep{nichol2022glide} make an attempt on both CLIP guidance and classifier-free guidance to enhance text-to-image generation.
Latent Diffusion Models (LDMs) \citep{rombach2022high} apply diffusion models on the latent space of powerful pretrained autoencoders for high-resolution synthesis. RPG \citep{yang2024mastering} proposes a LLM-grounded text-to-image diffusion and utilizes the multimodal chain-of-thought reasoning ability of MLLMs to enable complex/compositional image generation. 
Regarding text-to-video diffusion models, recent methods mainly leverage the pretrained text-to-image diffusion models in zero-shot \citep{qi2023fatezero,wang2023zero} and one-shot \citep{wu2022tune,liu2023video} methodologies for text-to-video editing. For example, Tune-A-Video \citep{wu2022tune} employs DDIM \citep{song2020denoising} inversion to provide structural guidance for sampling, and proposes efficient attention tuning for improving temporal consistency. FateZero \citep{qi2023fatezero} fuses the attention maps in the inversion process and generation process to preserve the motion and structure consistency during editing. In this work, we for the first time improve both text-to-image and text-to-video diffusion models with a general context-aware trajectory adapter.

\begin{figure*}[h]
\centering
% \raggedleft
% \flushleft
\vspace{-0.4in}
\includegraphics[width=0.99\linewidth]{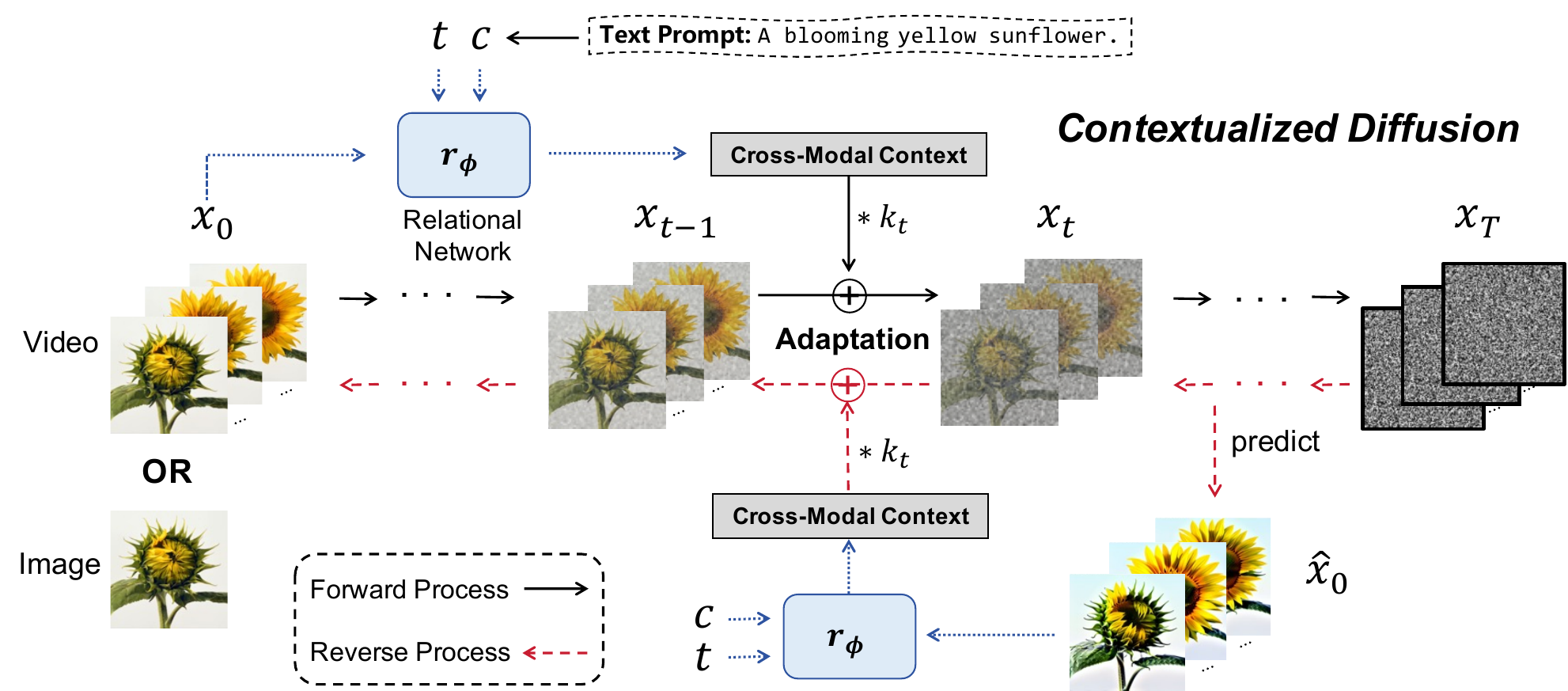}
% \vspace{-0.05in}
\caption{Illustration of our \method.}
\vskip -0.2in
\label{fig:method}
\end{figure*}

\paragraph{Diffusion Trajectory Optimization}
% The trajectory of a diffusion model denotes the distribution of the entire diffusion process, which is defined by the forward process that convolves the data distribution with a noise distribution. 
Our work focuses on optimizing the diffusion trajectories that denotes the distribution of the entire diffusion process.
% Optimizing the diffusion trajectory can improve sampling efficiency and sample quality. 
Some methods modify the forward process with a carefully-designed transition kernel or a new data-dependent initialization distribution \citep{liu2022flow,hoogeboom2022blurring,dockhorn2021score,lee2021priorgrad,karras2022elucidating}. For example, Rectified Flow \citep{liu2022flow} learns a straight path connecting the data distribution and prior distribution. 
Grad-TTS \citep{popov2021grad} and PriorGrad \citep{lee2021priorgrad} introduce conditional forward process with data-dependent priors for audio diffusion models.
Other methods mainly parameterize the forward process with additional neural networks \citep{zhang2021diffusion,kim2022maximum,kingma2021variational}. VDM \citep{kingma2021variational} parameterizes the noise schedule with a monotonic neural network, which is jointly trained with the denoising network. 
However, these methods only utilize unimodal information in forward process, and thus are inadequate for handling complex multimodal synthesis tasks.
In contrast, our \method for the first time incorporates cross-modal context into the diffusion process for improving text-guided visual synthesis, which is more informative and contextual guidance compared to text guidance. 

\section{Method}

\subsection{Cross-Modal Contextualized Diffusion}
We aim to incorporate cross-modal context of each text-image(video) pair $(\vc,\vx_0)$ into the diffusion process as in \cref{fig:method}.
We use clip encoders to extract the embeddings of each pair, and adopt an relational network (\textit{e.g.}, cross attention) to model the interactions and alignments between the two modalities as cross-modal context. This context is then propagated to all timesteps of the diffusion process as a bias term (we highlight the critical parts of our \method in \textcolor{brown}{brown}):
\begin{equation}
    q_{\phi}(\vx_t|\vx_0,\vc)=\mathcal{N}(\vx_t,\sqrt{\Bar{\alpha}_t}\vx_0+\textcolor{brown}{k_t\vr_\phi(\vx_0,\vc,t)},(1-\Bar{\alpha}_t)\mI),
    \label{eq-ddpm-forward1}
\end{equation}
where scalar $k_t$ control the magnitude of the bias term, and we set the $k_t$ to $\sqrt{\Bar{\alpha}_t} \cdot (1-\sqrt{\Bar{\alpha}_t})$. $\vr_\phi(\cdot)$ is the relational network with trainable parameters $\phi$, it takes the visual sample $\vx_0$ and text condition $\vc$ as inputs and produces the bias with the same dimension as $\vx_0$. 
% \se{the output needs to have the same dimension as x though? might want to say what kind of architecture you use} done
% We don't add any restriction on $\vr_\phi(\cdot)$, and it can be arbitrary form of neural network. 
% \se{a bit strange to have the kt scaling since r has an explicit t dependence} done
% With instance-level adaptation, we can specialize the diffusion trajectory for each instance, and thus increase the maneuverability of conditional generation and sample quality.
% In this way, it enables the forward process more aligned with the reverse process, which benefits the model convergence.
% We do not make any restriction on the function forms of adapter term $\vr_\phi(\vx_0,c,t)$, in order to optimize the entire trajectory of process. 

Concretely, the forward process is defined as  $q_\phi(\vx_1,\vx_2,...,\vx_T|\vx_0,\vc) = \prod_{t=1}^T q_\phi(\vx_t|\vx_{t-1},\vx_0,\vc)$. Given cross-modal context $\vr_\phi(\vx_0,\vc,t)$, the forward transition kernel depends on $\vx_{t-1},\vx_0,$ and $\vc$ :
\begin{equation}
    q_{\phi}(\vx_{t}|\vx_{t-1},\vx_0,\vc)=\mathcal{N}(\sqrt{\alpha_t}\vx_{t-1}+\textcolor{brown}{k_t\vr_\phi(\vx_0,\vc,t)}-\sqrt{\alpha_{t}}\textcolor{brown}{k_{t-1}\vr_\phi(\vx_0,\vc,t-1)},\beta_t\mI),
    \label{eq-ddpm-forward2}
\end{equation}
where $\beta_t = 1-\alpha_t$. This transition kernel gives marginal distribution as \cref{eq-ddpm-forward1} (proof in \cref{prof_process}). 
% \se{maybe refer to ddim, people might not be familiar with this kind of kernel. probably should say explicitly that this conditional gives the marginal (5)} done
% \se{beta is not defined?} done
% where $\vb_0$ is set to 0.
At each timestep $t$, we add a noise that explicitly biased by the cross-modal context.
% and thus the trajectory will be well adapted to fit each text-image(video) pair. 
With \cref{eq-ddpm-forward1} and \cref{eq-ddpm-forward2}, we can derive the posterior distribution of the forward process for $t>1$ (proof in \cref{prof_process}):

\begin{small}
\begin{equation}
    q_{\phi}(\vx_{t-1}|\vx_t,\vx_0,\vc)=\mathcal{N}(\frac{\sqrt{\Bar{\alpha}_{t-1}}\beta_t}{1-\Bar{\alpha}_t}\vx_0+\frac{\sqrt{\alpha_t}(1-\Bar{\alpha}_{t-1})}{1-\Bar{\alpha_t}}(\vx_t-\textcolor{brown}{\vb_t(\vx_0,\vc)})+\textcolor{brown}{\vb_{t-1}(\vx_0,\vc)},\frac{(1-\Bar{\alpha}_{t-1})\beta_t}{1-\Bar{\alpha}_t}\mI),
    \label{eq-ddpm-backward}
\end{equation}
\end{small}

where {$\vb_t(\vx_0,\vc)$} is an abbreviation form of {$k_t\vr_\phi(\vx_0,\vc,t)$}, and we use it for simplicity. With \cref{eq-ddpm-backward}, we can simplify the training objective which will be described latter.
% With our context-aware diffusion, both forward and reverse processes explicitly utilize the text condition, facilitating the model capacity of conditional modeling.
In this way, we contextualize the entire diffusion process with a context-aware trajectory adapter.
In \method, we also utilize our context-aware context to adapt the reverse process of diffusion models, which encourages to align with the adapted forward process, and facilitates the precise expression of textual semantics in visual sampling process.

% \se{might want to give reader some guidance about why you are deriving these things. (5) allows you efficient sampling in denoising score matching (without having to simulate the entire chain). (7) allows you to simplify the variational optimization later} done

\subsection{Adapting Reverse Process} 
We aim to learn a contextualized reverse process $\{p_{\theta}(\vx_{t-1}|\vx_t,\vc)\}_{t=1}^T$ , which minimizes a variational upper bound of the negative log likelihood. $p_{\theta}(\vx_{t-1}|\vx_t,\vc)$ is gaussian kernel with learnable mean and pre-defined variance. Allowing the forward transition kernel to depend on $\vx_0$ and $\vc$, the objective function $\mathcal{L}_{\theta,\phi}$ of our \method can be formulated as (proof in \cref{prof_objective}):
\begin{equation}
    \begin{aligned}
        \mathcal{L}_{\theta,\phi} = \E_{{q_\phi(\vx_{1:T}|\vx_0,\vc)}}& \bigg[\KL(q_\phi(\vx_T|\vx_0,\vc) \Vert p(\vx_T|\vc))-\log p_{\theta}(\vx_0|\vx_1,\vc) \\&+\sum_{t>1}\KL(q_\phi(\vx_{t-1}|\vx_{t},\vx_0,\vc) \Vert p_{\theta}(\vx_{t-1}|\vx_t,\vc) )\bigg],
    \end{aligned}
    \label{eq-training-ori}
\end{equation}
% \se{expand the expectation notation} not sure
where $\theta$ denotes the learnable parameters of denoising network in reverse process. \cref{eq-training-ori} uses KL
divergence to directly compare $p_{\theta}(\vx_{t-1}|\vx_t,\vc)$ against the adapted forward process posteriors, which are tractable
when conditioned on $\vx_0$ and $\vc$. 
If $\vr_\phi$ is identically zero, the objective can be viewed as the original DDPMs. Thus \method is theoretically capable of achieving better likelihood compared to original DDPMs.
% \se{compared to.. you might want to say that if r is identically zero, then this gives the original ddim/ddpm elbo.} done

% \se{i would add a bit more text on how phi is optimized. now q depends on phi, so you need either likelihood rations or reparamterization trick (i assume you use the latter?)} done

Kindly note that optimizing $\mathcal{L}_t= E_{q_\phi}\KL(q_\phi(\vx_{t-1}|\vx_{t},\vx_0,\vc) \Vert p_{\theta}(\vx_{t-1}|\vx_t,\vc))$ is equivalent to matching the means for $q_\phi(\vx_{t-1}|\vx_{t},\vx_0,\vc)$ and $p_{\theta}(\vx_{t-1}|\vx_t,\vc)$, as they are gaussian distributions with the same variance. 
% \se{again, careful q (which you take an expectation over) depends on phi too} done 
According to \cref{eq-ddpm-backward}, directly matching the means requires to parameterize a neural network $\vmu_\theta$ that not only predicting $\vx_0$, but also matching the complex cross-modal context information in the forward process, i.e., 
\begin{equation}
    \mathcal{L}_{\theta,\phi,t} = \big|\big|\vmu_\theta(\vx_t,c,t)-\frac{\sqrt{\Bar{\alpha}_{t-1}}\beta_t}{1-\Bar{\alpha}_t}\vx_0-\frac{\sqrt{\alpha_t}(1-\Bar{\alpha}_{t-1})}{1-\Bar{\alpha_t}}(\vx_t-\textcolor{brown}{\vb_t(\vx_0,\vc)})-\textcolor{brown}{\vb_{t-1}(\vx_0,\vc)}\big|\big|^2_2
\end{equation}
% \se{maybe refer to the previous gaussian derivation, so it's clear where they are used?} done
\paragraph{Simplified Training Objective} Directly optimizing this objective is inefficient in practice because it needs to compute the bias twice at each timestep.
To simplify the training process, we employ a denoising network $\vf_\theta(\vx_t,\vc,t)$ to directly predict $\vx_0$ from $\vx_t$ at each time step t, and insert the predicted $\hat{\vx}_0$ in \cref{eq-ddpm-backward}, i.e., $p_{\theta,\phi}(\vx_{t-1}|\vx_t,\vc) = q_\phi(\vx_{t-1}|\vx_{t},\hat{\vx}_0,\vc)$. Under mild condition, we can derive that the reconstruction objective $\E||f_\theta-\vx_0||^2_2 $ is an upper bound of $\mathcal{L}_t$, and thus an upper bound of negative log likehood (proof in \cref{prof_objective}).
% Here we assume the relational network $\vr_\phi(\vx_0,\vc,t)$ is Lipschitz continuous \citep{neyshabur2017exploring}, so we can prove that the re-construction loss $\E||f_\theta-\vx_0||_2 $ is an upper bound of $\mathcal{L}_t$, scaled by a factor (proof in \cref{prof_objective}).
% \se{didn't understand this part} done
Our simplified training objective is: 
\begin{equation}
    \mathcal{L}_{\theta,\phi} = \sum_{t=1}^T\lambda_t\E_{\vx_0,\vx_t}||\vf_\theta(\vx_{t,\phi},\vc,t)-\vx_0||^2_2,\label{eq_objective_final}
\end{equation}
where $\lambda_t$ is a weighting scalar. 
We set $k_T = 0$ and there is no learnable parameters in $\KL(q(\vx_T|\vx_0,\vc) \Vert p(\vx_T|\vc))$, which can be ignored. To adapt the reverse process at each timestep, we can efficiently sample a noisy sample $\vx_t$ according to \cref{eq-ddpm-forward1} using re-parameterization trick, which has included parameterized cross-modal context {$k_t\vr_\phi(\vx_0,\vc,t)$}, and then passes $\vx_t$ into the denoising network. The gradients will be propagated to $\vr_\phi$ from the denoising network, and our context-aware adapter and denoising network are jointly optimized in training.

% \se{how do you handle the prior (time T)?}

\paragraph{Context-Aware Sampling} During sampling, we use the denoising network to predict $\hat{\vx}_0$, and the predicted context-aware adaptation $\vr_\phi(\hat{\vx}_0,\vc,t)$ is then used to contextualize the sampling trajectory. Hence the gaussian kernel $p_\theta(\vx_{t-1}|\vx_t, \vc)$ has mean: 
\begin{equation}
    \begin{aligned} \frac{\sqrt{\Bar{\alpha}_{t-1}}\beta_t}{1-\Bar{\alpha}_t}\hat{\vx}_0+\frac{\sqrt{\alpha_t}(1-\Bar{\alpha}_{t-1})}{1-\Bar{\alpha_t}}(\vx_t-\textcolor{brown}{\vb_t(\hat{\vx}_0,\vc)})+\textcolor{brown}{\vb_{t-1}(\hat{\vx}_0,\vc)}, 
        % \quad
        % \text{variance}: \frac{(1-\Bar{\alpha}_{t-1})\beta_t}{1-\Bar{\alpha}_t}\mI
    \end{aligned}
\end{equation}
% \se{confused about this paragraph. if (11) is the conditional for the decoder/p distribution, then why not plug it into the KL above?} done
where {$\vb_t(\hat{\vx}_0,\vc)$} = {$k_t\vr_\phi(\hat{\vx}_0,c,t)$}, and variance $\frac{(1-\Bar{\alpha}_{t-1})\beta_t}{1-\Bar{\alpha}_t}\mI$. In this way, our \method can effectively adapt sampling process with cross-modal context, which is more informative and contextual guided compared to traditional text guidance \citep{rombach2022high,saharia2022photorealistic}. 
Next, we will introduce how to generalize our contextualized diffusion to DDIMs for fast sampling.

% \quad Towards a better utilization of the information generated in the sampling process, we adopt the Self-conditioning technique \citep{chen2022analog}. In standard diffusion sampling, at each time step t the denoising network predict $\vx_0$ using only $\vx_t$ as input, and the predicted $\hat{\vx}_0^t$ is disgard in the next sampling step. Here we adopt the Self-conditioning technique that the denoising network at time t additionally take $\hat{\vx}_0^{t+1}$ as input, i.e., the prediction of $\vx_0$ from time $t+1$. This sampling process can be viewed as a progressively refinement of $\vx_0$ estimates, where at time t, $\vx_0$ is predicted as: $\hat{\vx}_0^t = \vf_\theta(\vx_t,\hat{\vx}_0^{t+1},\vc,t)$. For efficient training, with half
% probability we set $\hat{\vx}_0^{t+1}=0$ to train the denoising network $\vf_\theta(\vx_t,0,\vc,t)$ without self-conditioning. Otherwise, $\hat{\vx}_0^{t+1}$ is first estimated by $\vf_\theta(\vx_{t+1},0,\vc,t+1)$ and then used for self-conditioning training. Note that under the second circumstance, we do not backpropagate through the estimated $\hat{\vx}_0^{t+1}$.

\begin{figure*}[ht]
\vspace{-6mm}
\centering
\includegraphics[width=0.9\textwidth]{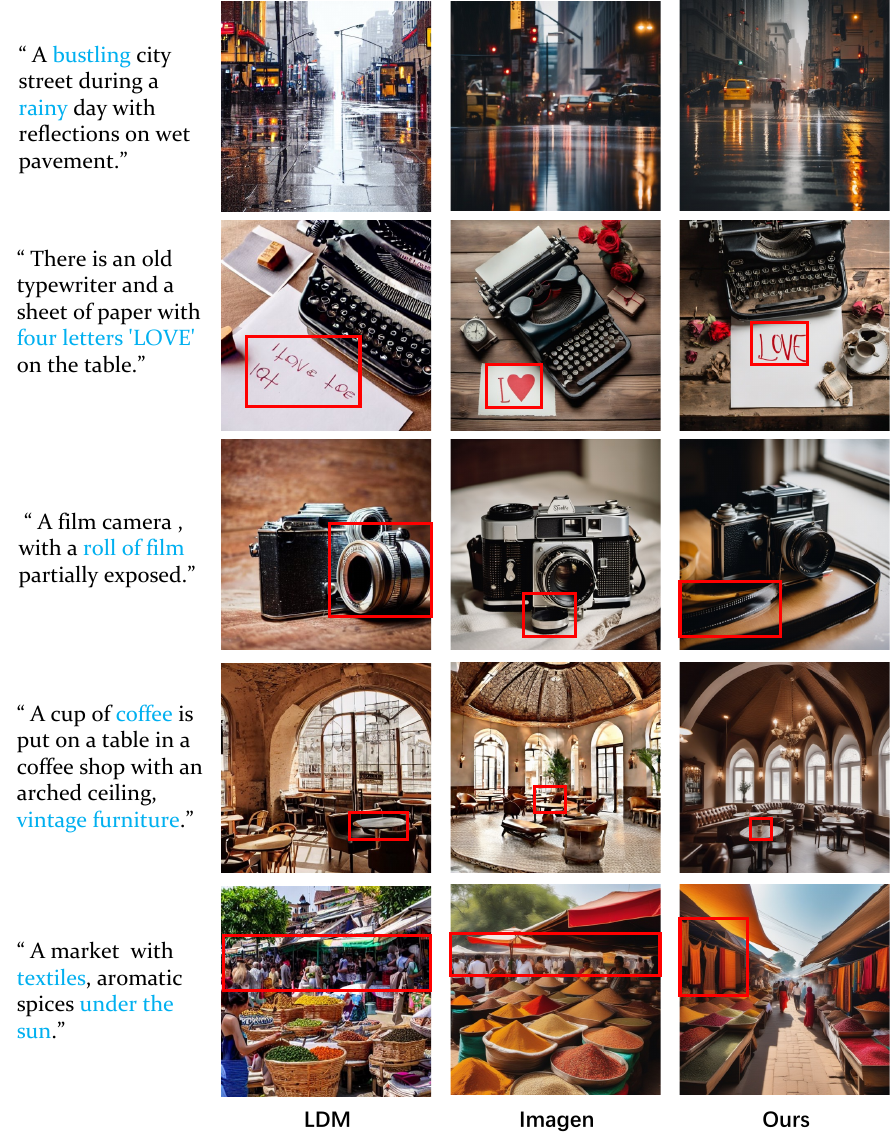}
\vspace{-5mm}
\caption{Synthesis examples demonstrating text-to-image capabilities of for various text prompts with LDM, Imagen, and
\method (Ours). Our model can better express the semantics of the texts marked in blue. We use \textcolor{red}{red} boxes to highlight critical fine-grained parts where LDM and Imagen fail to align with texts. For example, in second row, only our method successfully generates the four letters spelling "LOVE". In third row, we generate the specific detail of a film roll, while other methods lose this detail.}
% \vspace{-4mm}
\label{app-fig-more-image}
\end{figure*}

\subsection{Generalizing Contextualized Diffusion to DDIMs}
DDIMs \citep{song2020denoising} accelerate the reverse process of pretrained DDPMs, 
% DDIM inversion that reverses DDIM sampling \citep{song2020denoising} is widely utilized in visual editing tasks for preserving structural information.
which are also faced with the inconsistency problem that exists in DDPMs.
Therefore, we address this problem by generalizing our contextualized diffusion to DDIMs.
Specifically, we define a posterior distribution $q_\phi(\vx_{t-1}|\vx_t,\vx_0,\vc)$ for each timestep, thus the forward diffusion process has the desired distribution:
\begin{equation}
    q_\phi(\vx_t|\vx_0,\vc)=\mathcal{N}(\vx_t,\sqrt{\Bar{\alpha}_t}\vx_0+\textcolor{brown}{\vb_t(\vx_0,\vc)},(1-\Bar{\alpha}_t)\mI),
\end{equation}
If the posterior distribution is defined as:
\begin{equation}
    q(\vx_{t-1}|\vx_t,\vx_0,\vc) = \mathcal{N}(\sqrt{\bar{\alpha}_{t-1}}\vx_0+\sqrt{1-\bar{\alpha}_{t-1}-\sigma^2_t}*\frac{\vx_t-\sqrt{\bar{\alpha}_t}\vx_0}{\sqrt{1-\bar{\alpha}_t}},\sigma^2_t\mI),
\end{equation}
then the mean of $q_\phi(\vx_{t-1}|\vx_0,\vc)$ is (proof in \cref{prof_process}): 
\begin{equation}
    \begin{aligned}
       \sqrt{\bar{\alpha}_{t-1}}\vx_0+\textcolor{brown}{\vb_t(\vx_0,\vc)}*\frac{\sqrt{1-\bar{\alpha}_{t-1}-\sigma^2_t}}{\sqrt{1-\bar{\alpha}_t}}
        % &\text{variance}: (1-\bar{\alpha}_{t-1})*\mI
    \end{aligned}
\end{equation}

To match the forward diffusion, we need to replace the adaptation $k_t\vr_\phi(\vx_0,\vc,t)*\frac{\sqrt{1-\bar{\alpha}_{t-1}-\sigma^2_t}}{\sqrt{1-\bar{\alpha}_t}}$with $k_{t-1}\vr_\phi(\vx_0,\vc,t-1)$.
% Naively, we can change the mean of posterior distribution by minus $k_t\vr_\phi(\vx_0,c,t)*\frac{\sqrt{1-\bar{\alpha}_{t-1}-\sigma^2_t}}{\sqrt{1-\bar{\alpha}_t}}$ and add $k_{t-1}\vr_\phi(\vx_0,c,t-1)$. 
Given $\sigma^2_t =0$, the sampling process becomes deterministic:
\begin{equation}
    \begin{aligned}
        \Tilde{\vx}_{t-1} &= \sqrt{\bar{\alpha}_{t-1}}\hat{\vx}_0+\sqrt{1-\bar{\alpha}_{t-1}}*\frac{\vx_t-\sqrt{\bar{\alpha}_t}\hat{\vx}_0}{\sqrt{1-\bar{\alpha}_t}}\\
        \vx_{t-1} &= \Tilde{\vx}_{t-1}-\textcolor{brown}{\vb_t(\hat{\vx}_0,\vc)}*\frac{\sqrt{1-\bar{\alpha}_{t-1}}}{\sqrt{1-\bar{\alpha}_t}}+\textcolor{brown}{\vb_{t-1}(\hat{\vx}_0,\vc)}.
    \end{aligned}
\end{equation}
In this way, DDIMs can better convey textual semantics in generated samples when accelerating the sampling of pretrained DDPMs, which will be evaluated in later text-to-video editing task. 

% \newpage

\section{Experiments}
We conduct experiments on two text-guided visual synthesis tasks: text-to-image generation, and text-to-video editing. 
For the text-to-image generation task, we aim to generate images given textual descriptions. 
For the text-to-video editing task, we aim to edit videos given textual descriptions.
We evaluate the semantic alignment between text conditions and synthesis results in the two tasks with qualitative and quantitative comparisons, and evaluate temporal consistency on text-to-video editing task. We further evaluate the generalization ability and model convergence of our \method.
\subsection{Text-to-Image Generation}
\paragraph{Datasets and Metrics.}
Following \cite{rombach2022high,saharia2022photorealistic}, we use public LAION-400M \citep{schuhmann2021laion}, a
dataset with CLIP-filtered 400 million image-text pairs for training \method.
% MSCOCO dataset contains 82k images for training and 40k
% images for testing. 
% Each image in this dataset has five text
% descriptions.
We conduct evaluations with FID and CLIP score \citep{hessel2021clipscore,radford2021learning}, which aim to assess the generation quality and resulting image-text alignment. 
\paragraph{Implementation Details.} For our context-aware adapter, we use text CLIP and image CLIP \citep{radford2021learning} (ViT-B/32) to encode text and image inputs, and adopt multi-head cross attention \citep{vaswani2017attention} to model cross-modal interactions with 8 parallel attention layers. For the diffusion backbone, we mainly follow Imagen \citep{saharia2022photorealistic} using a $64\times64$ base diffusion model \citep{nichol2021improved,saharia2022palette} and a super-resolution diffusion
models to upsample a 64 × 64 generated image into a 256 × 256 image. 
For $64\times64\rightarrow 256\times256$ super-resolution, we use the efficient U-Net model in Imagen for improving memory efficiency.
We condition on the entire sequence of text embeddings \citep{raffel2020exploring} by adding cross attention
\citep{ramesh2022hierarchical} over the text embeddings at multiple resolutions. More details about the hyper-parameters can be found in \cref{app-hyperparameters}.
% The \method is trained by AdamW \citep{loshchilov2018decoupled} optimizer with a learning rate of $1e^{-4}$ and a batch size of 1024 on NVIDIA A100s. 
% \paragraph{Qualitative Results.}

\begin{table*}[ht]
\small
% \vskip -0.2in
\caption{\textbf{Quantitative results in text-to-image generation} with FID score on MS-COCO dataset for 256 × 256 image resolution.}
\vskip -0.15in
\label{tab:eval_coco}
\vspace{5pt}
\centering
\label{tab:zero_shot_mscoco}
\begin{tabular}{l|c| c c}
\toprule
\multirow{2}{*}{\bfseries{Approach}} & \multirow{2}{*}{\bfseries{Model Type}} &
\multirow{2}{*}{\bfseries{FID-30K}} & \bfseries{Zero-shot} \\
 & & & \bfseries{FID-30K} \\
\midrule
% AttnGAN \citep{xu2018attngan} & GAN & 35.49 & - \\
% DM-GAN \citep{zhu2019dm} & GAN & 32.64 & - \\
DF-GAN \citep{tao2022df} &  GAN & 21.42 & - \\
DM-GAN + CL \citep{ye2021improving} &  GAN & 20.79 & - \\
% XMC-GAN \citep{zhang2021cross} & GAN & 9.33 & - \\
LAFITE \citep{zhou2022towards} & GAN & 8.12 & - \\
Make-A-Scene \citep{gafni2022make} & Autoregressive & 7.55 & -\\
\midrule
DALL-E \citep{pmlr-v139-ramesh21a} & Autoregressive & - & 17.89  \\
% LAFITE \citep{zhou2022towards} & GAN & - & 26.94 \\
Stable Diffusion \citep{rombach2022high}  & Continuous Diffusion & -  & 12.63 \\
% LDM + \textbf{Context Diffusion}  & Continuous Diffusion & -  &  \\
% \midrule
GLIDE \citep{pmlr-v162-nichol22a} & Continuous Diffusion & -  & 12.24 \\
DALL-E 2 \citep{ramesh2022hierarchical} & Continuous Diffusion & -  & 10.39 \\
Improved VQ-Diffusion \citep{tang2022improved} & Discrete Diffusion&-&8.44\\
Simple Diffusion \citep{hoogeboom2023simple}&Continuous Diffusion&-&8.32\\
Imagen \citep{saharia2022photorealistic}  & Continuous Diffusion & -  & 7.27 \\
Parti \citep{yu2022scaling} & Autoregressive & -  & 7.23\\
Muse \citep{chang2023muse}&Non-Autoregressive&-&7.88\\
eDiff-I \citep{balaji2022ediffi}& Continuous Diffusion&-&6.95\\
ERNIE-ViLG 2.0 \citep{feng2023ernie}&Continuous Diffusion&-&6.75\\
RAPHAEL \citep{xue2023raphael} &Continuous Diffusion&-&6.61\\
% \midrule
% LDM \citep{rombach2022high}  & Continuous Diffusion & -  & 12.63 \\
% LDM + \textbf{Context Diffusion}  & Continuous Diffusion & -  &  \\
% \midrule
% Imagen \citep{saharia2022photorealistic}  & Continuous Diffusion & -  & 7.27 \\
% Imagen + \textbf{Context Diffusion}   & Continuous Diffusion & -  &  \\
\midrule
% \rowcolor{LightCyan}
$\textbf{\method}$ & Continuous Diffusion & - & \textbf{6.48} \\
\bottomrule
\end{tabular}
\vskip -0.05in
\end{table*}

\begin{table}
\small
\caption{\textbf{Quantitative results in text-to-video editing.} Text. and Temp. denote CLIP-text and CLIP-temp, respectively. User study shows the preference rate of \method against baselines via human evaluation.}
\vskip -0.1in
\label{tb:baseline camparision}
\centering
\begin{tabular}{lcccccccccc} 
\toprule
&\multicolumn{2}{c}{Metric} &  \multicolumn{2}{c}{User Study} \\
\cmidrule(lr){2-3} \cmidrule(lr){4-5} 
Method &  Text.$\uparrow$ & Temp.$\uparrow$ & Text. ($\%$)$\uparrow$&Temp.($\%$)$\uparrow$  \\
\midrule
% Stable Diffusion~\cite{rombach2022high} &\textbf{0.278}&0.895&44&75\\
% % ControlNet~\cite{zhang2023adding} & 0.261&0.905&0.627&52&75&71&86\\
% \midrule
Tune-A-Video~\citep{wu2022tune} & 0.260 & 0.934 & 91& 84\\
FateZero~\citep{qi2023fatezero} & 0.252 & 0.954  & 84& 75\\
ControlVideo~\citep{zhao2023controlvideo}&0.258 & 0.961&81& 73\\
\midrule 
\textbf{\method} & \textbf{0.274} & \textbf{0.970} & - &- \\
\bottomrule
\end{tabular}
\vskip -0.2in
\label{tab-video-quantitative}
\end{table}

\paragraph{Quantitative and Qualitative Results} Following previous works \citep{rombach2022high,ramesh2022hierarchical,saharia2022photorealistic}, we make quantitative evaluations \method on the MS-COCO dataset using zero-shot FID score, which measures the quality and diversity of generated images. Similar to \cite{rombach2022high,ramesh2022hierarchical,saharia2022photorealistic}, 30,000 images are randomly selected from
the validation set for evaluation. As demonstrated in \cref{tab:eval_coco}, our \method achieves a new state-of-the-art performance on text-to-image generation task with 6.48 zero-shot FID score, outperforming previous dominant diffusion models such as Stable Diffusion \citep{rombach2022high}, DALL-E 2 \citep{ramesh2022hierarchical}, and Imagen \citep{saharia2022photorealistic}.  
We also make qualitative comparisons in \cref{app-fig-more-image}, and find that our \method can achieve more precise semantic alignment between text prompt and generated image than previous methods, demonstrating the effectiveness of incorporating cross-modal context into diffusion models. 

\begin{figure*}[ht]
\centering
% \vskip -0.4in
\includegraphics[width=1.\textwidth]{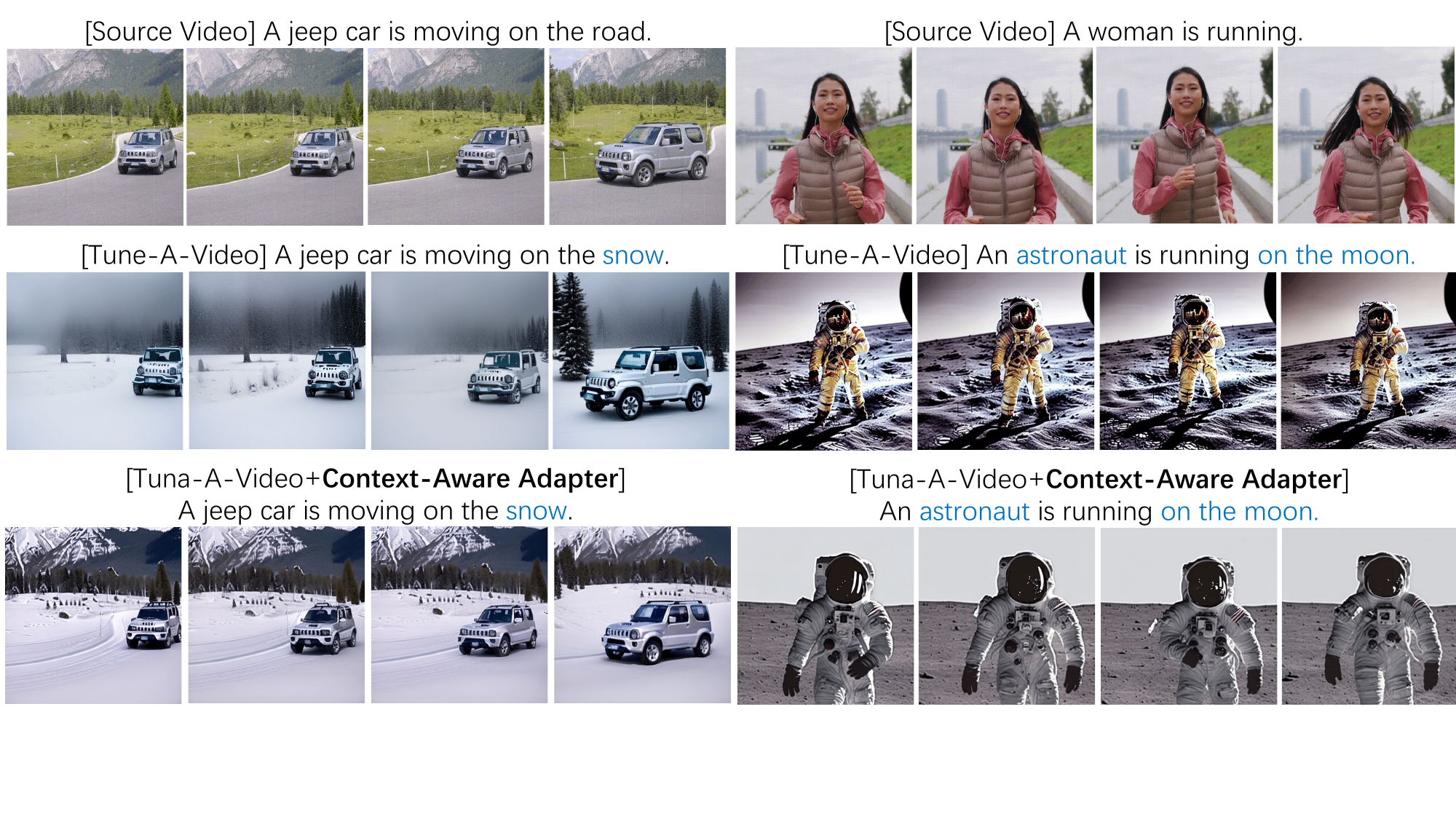}
\vskip -0.4in
\caption{Generalizing our context-aware adapter to Tune-A-Video \citep{wu2022tune}.}
\label{fig:qualitative-tune}
\vskip -0.2in
\end{figure*}

\subsection{Text-to-Video Editing}
\paragraph{Datasets and Metrics}
To demonstrate the strength of our \method for text-to-video edting, we use 42 representative videos taken from DAVIS dataset \citep{pont2017davis} and other in-the-wild videos following previous works \citep{wu2022tune,qi2023fatezero,bar2022text2live,esser2023structure}.  These videos cover a range of categories including animals, vehicles, and humans. 
To obtain video footage, we use BLIP-2 \citep{li2023blip} for automated captions. We also use their designed prompts for each video, including object editing, background changes, and style transfers. 
To measure textual alignment, we compute
average CLIP score between all frames of output videos
and corresponding edited prompts.
For temporal consistency, we compute
CLIP \citep{radford2021learning} image embeddings on all frames of output videos
and report the average cosine similarity between all pairs of
video frames. 
Moreover, We perform user study to quantify text alignment, and temporal consistency by pairwise comparisons between the baselines and our \method. A total of 10 subjects participated in this user study. Taking text alignment as an
example, given a source video, the participants are instructed to select which edited video is more
aligned with the text prompt in the pairwise comparisons between the baselines and \method.
\paragraph{Implementation Details} In order to reproduce the baselines of Tune-A-Video \citep{wu2022tune}, FateZero \citep{qi2023fatezero}, and ControlVideo \citep{zhao2023controlvideo}, we use their official repositories for one-shot video tuning. Following FateZero, we use
the trained Stable Diffusion v1.4 \citep{rombach2022high} as the base text-to-image diffusion model, and fuse the attention maps in DDIM inversion \citep{song2020denoising} and sampling processes for retaining both structural and motion information. We
fuse the attentions in the interval of $t \in [0.5 \times T, T]$ of the
DDIM step with total timestep T = 20. For context-aware adapter, we use the same encoders and cross attention as in text-to-image generation. We additionally incorporate spatio-temporal attention, which includes spatial self-attention and temporal causal attention, into our context-aware adapter for capturing spatio-temporal consistency. For each source video, we tune our adapter using source text prompt for learning both context-aware structural and motion information, and use the learned adapter to conduct video editing with edited text prompt. Details about the hyper-parameters are in \cref{app-hyperparameters}.
% \paragraph{ Results}

\paragraph{Quantitative and Qualitative Results}
We report our quantitative and qualitative results
in \cref{tab-video-quantitative} and \cref{fig-video-qualitative}. 
Extensive results demonstrate that \method substantially outperforms all these baselines in both textual alignment and temporal consistency. 
% Although DDIM inversion utilizes textual information, our model incorporates high-order cross-modal context for more informative trajectory adaptation. 
Notably, in the textual alignment in user study, we outperform the baseline by a significant margin (over 80\%), demonstrating the superior cross-modal understanding of our contextualized diffusion. In qualitative comparisons, we observe that \method not only achieves better semantic alignment, but also preserves the structure information in source video. Besides, the context-aware adapter in our contextualized diffusion can be generalized to previous methods, which substantially improves the generation quality as in \cref{fig:qualitative-tune}. More results demonstrating our generalization ability can be found in \cref{app-sec-general}.

\subsection{Generalizing to Class-to-Image and Layout-to-Image Generation}
We generalize our context-aware adapter into class and layout conditional generation tasks. We replace the text encoder in original adapter with ResNet blocks for embedding classes or layouts, and keep the original image encoder and cross-attention module for obtaining cross-modal context information. We put both quantitative and qualitative results in \cref{table:class,table:layout} and \cref{app-class-to-image,app-layout-to-image}. From the results, we conclude that our context-aware adapter can benefit the conditional diffusion models with different condition modalities and enable more realistic and precise generation consistent with input conditions, demonstrating the satisfying generalization ability of our method. 
\begin{table}[ht]
\caption{Performance comparison in class-to-image generation on ImageNet 256$\times$256.}
\centering
\begin{tabular}{l ccccc} 
\toprule
 Method &  FID $\downarrow$ & IS $\uparrow$& Precision $\uparrow$ &Recall $\uparrow$ \\ 
\midrule
BigGAN \citep{brock2018large}&6.95&203.63&0.87&0.28\\
ADM-G \citep{dhariwal2021diffusion}&4.59& 186.70 &0.82& 0.52\\
LDM \citep{rombach2022high}&3.60& 247.67& 0.87& 0.48\\
LDM+\textbf{Context-Aware Adapter}&\textbf{2.97}&\textbf{273.04}&\textbf{0.89}&\textbf{0.55}\\
\bottomrule
\end{tabular}
\label{table:class}
% \vspace{-0.2in}
\end{table}

\begin{table}[ht]
\caption{FID performance comparison in layout-to-image generation on MS-COCO 256$\times$256.}
\centering
\begin{tabular}{l c} 
\toprule
 Method &  FID$\downarrow$ \\ 
\midrule
VQGAN+T \citep{jahn2021high}&56.58\\
Frido \citep{fan2023frido}&37.14 \\
LDM \citep{rombach2022high}&40.91\\
LDM+\textbf{Context-Aware Adapter}&\textbf{34.58}\\
\bottomrule
\end{tabular}
\label{table:layout}
% \vspace{-0.2in}
\end{table}

\begin{figure*}[ht]
\vspace{-6mm}
\centering
\includegraphics[width=1.\textwidth]{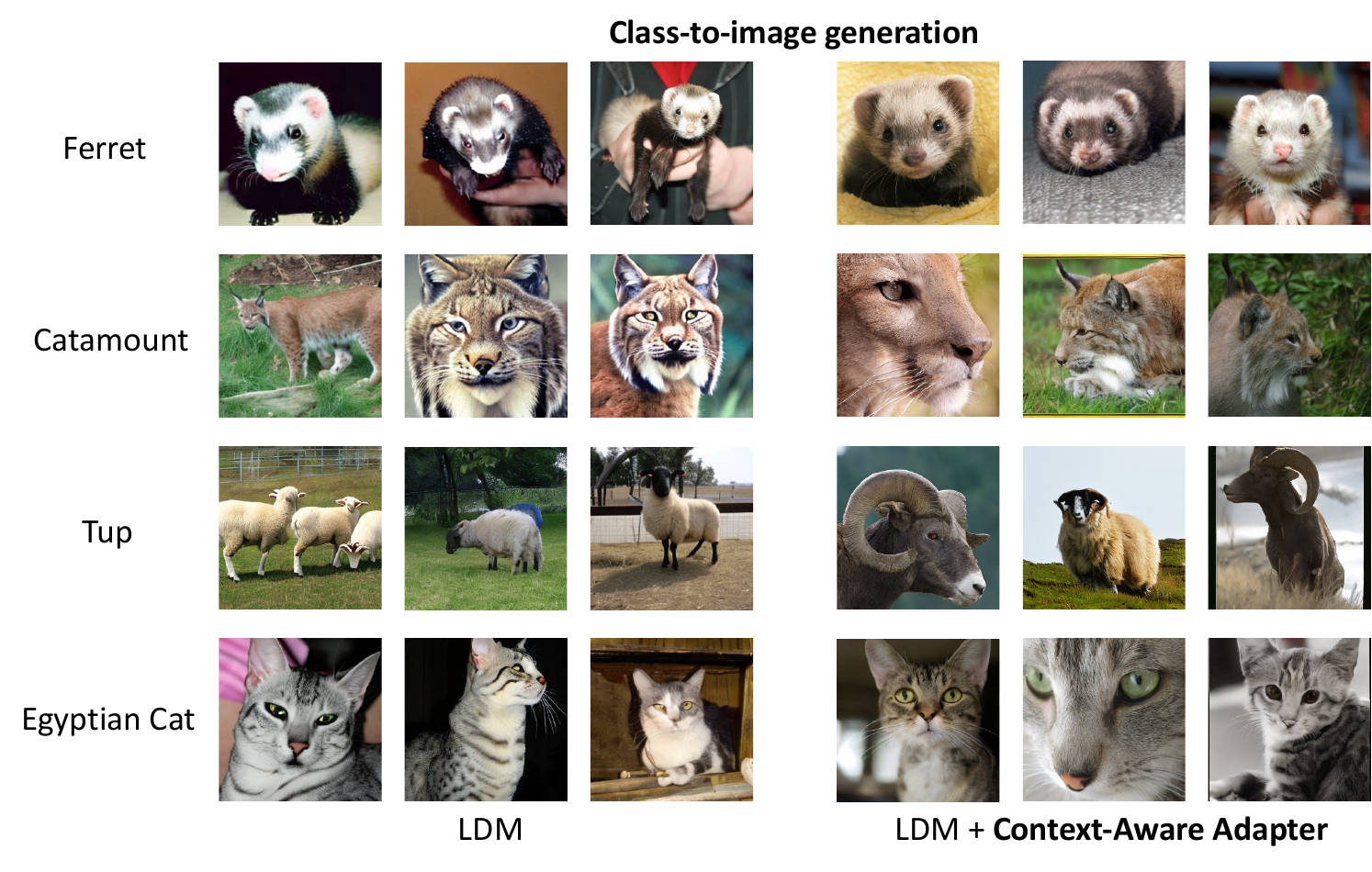}
\caption{Qualitative results in class-to-image generation on ImageNet 256$\times$256. Our context-aware adapter improves the generation quality of LDM.}
\label{app-class-to-image}
% \vspace{-8mm}
\end{figure*}

\begin{figure*}[ht]
\centering
\includegraphics[width=0.9\textwidth]{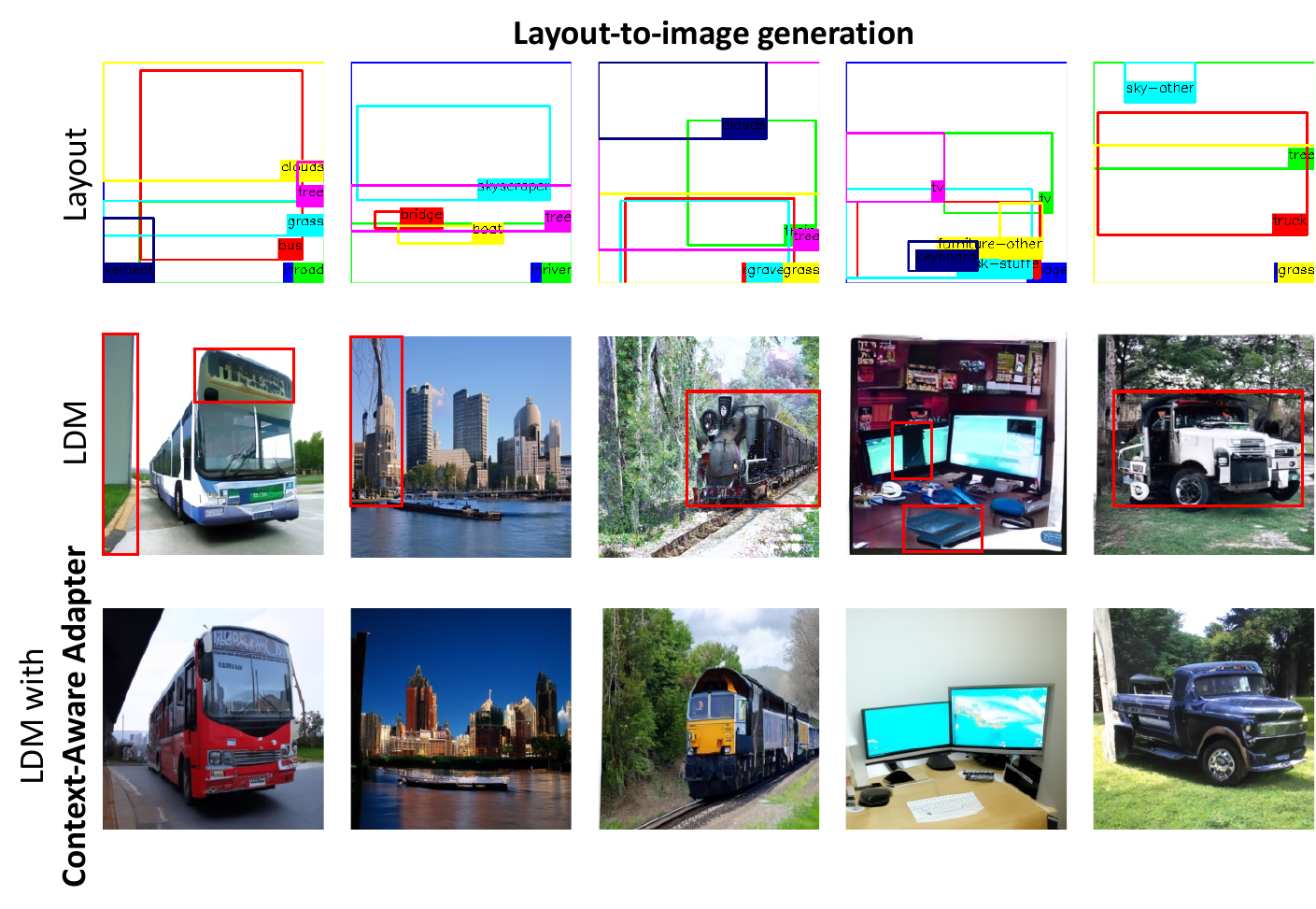}
\caption{Qualitative results in layout-to-image generation on MS-COCO 256$\times$256. Our context-aware adapter improves both fidelity and precision of the generation results of LDM. We use \textcolor{red}{red} boxes to highlight critical fine-grained parts where original LDM fails to align with conditional layout. Our method substantially improves both quality and precision of the generation results.}
\label{app-layout-to-image}
% \vspace{-8mm}
\end{figure*}

\subsection{Model Analysis}

\paragraph{Visual Analysis on Context Awareness of Our Model}
We conduct visual analysis to investigate how our context-aware adapter works in text-guided visual synthesis. As illustrated in \cref{app-video-heatmap-flower} and \cref{app-video-heatmap-snow}, we visualize the heatmaps of text-image cross-attention module in the sampling process of each frame image. We find that our context-aware adapter can enable the model to better focus on the fine-grained semantics in text prompt and sufficiently convey them in final generation results. Because incorporating textual information into diffusion process of the image benefits the cross-modal understanding for image diffusion models.
\begin{figure*}[ht]
\vspace{-4mm}
\centering
\includegraphics[width=1\textwidth]{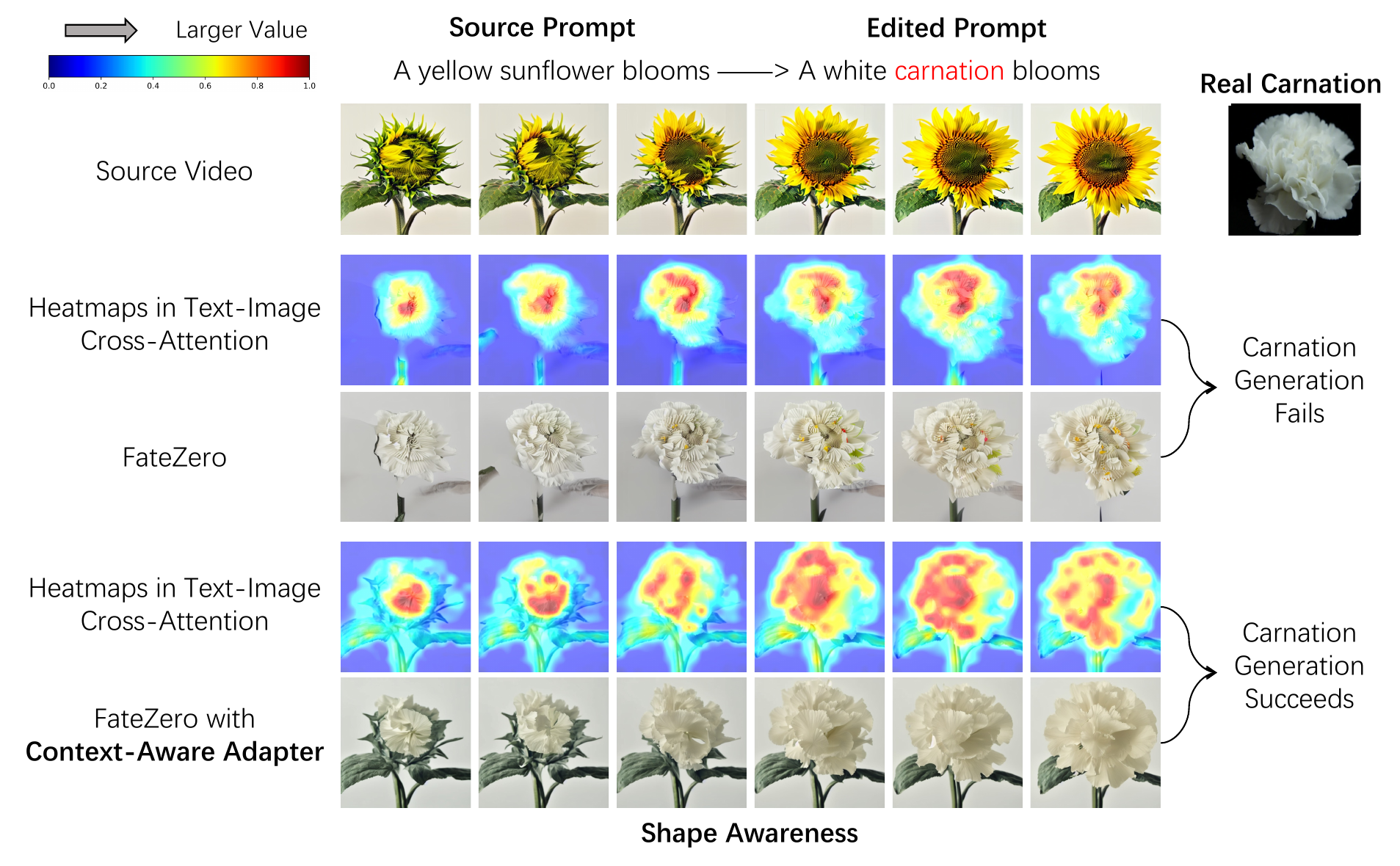}
\vspace{-8mm}
\caption{Our context-aware adapter improves the shape awareness of diffusion models in text-guided video editing.}
\label{app-video-heatmap-flower}
\vspace{-6mm}
\end{figure*}
\begin{figure*}[ht]
\vspace{-4mm}
\centering
\includegraphics[width=1\textwidth]{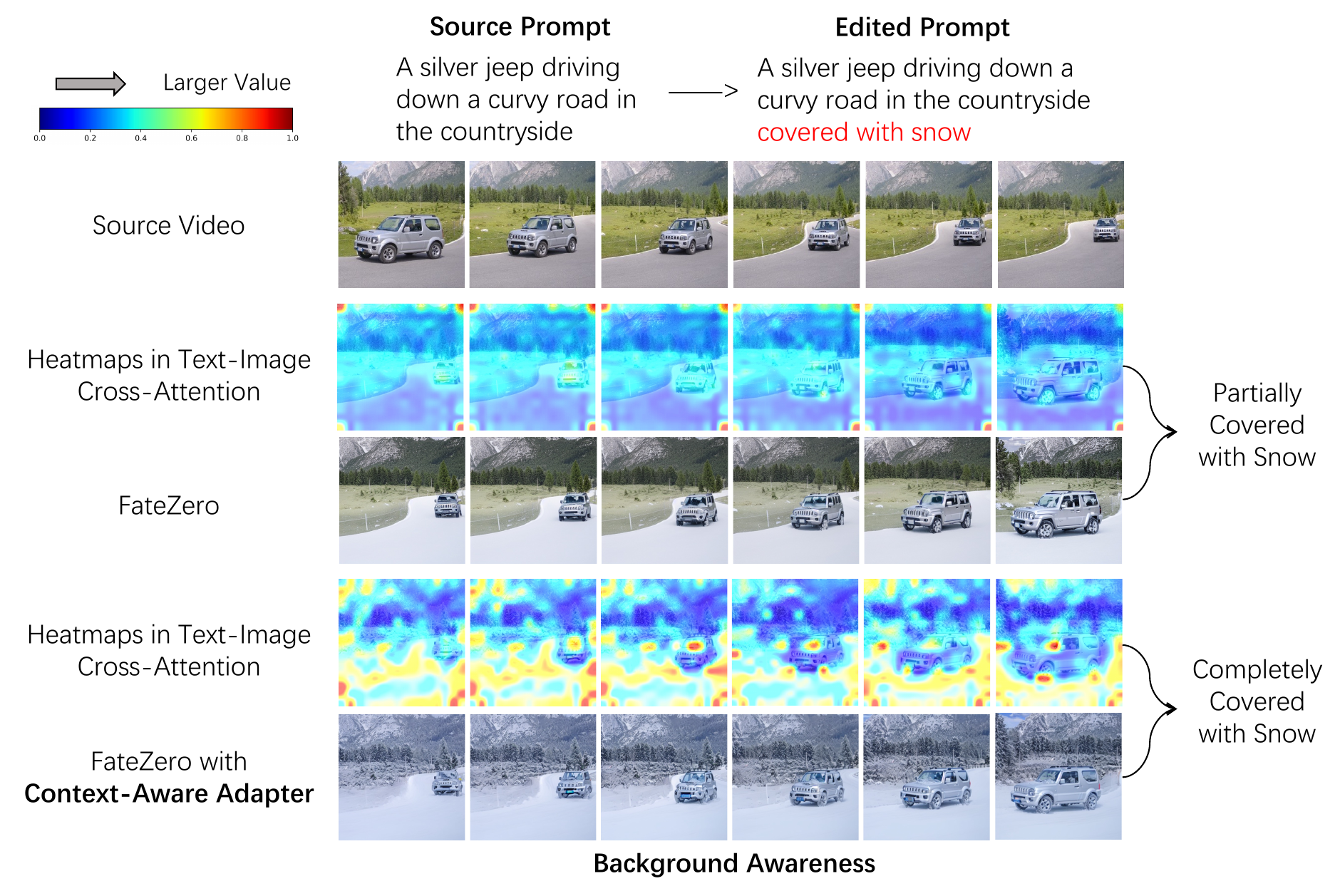}
\caption{Our context-aware adapter can improve the background awareness of diffusion models in text-guided video editing.}
\label{app-video-heatmap-snow}
% \vspace{-8mm}
\end{figure*}

\paragraph{Guidance Scale \textit{vs}.\ FID} Given
the significance of classifier-free guidance weight in controlling image quality and text alignment, in \cref{fig:tradeoff}, we
conduct ablation study on the trade-off between CLIP and FID scores across a range of guidance
weights, specifically 1.5, 3.0, 4.5, 6.0, 7.5, and 9.0. The results indicate that our context-aware adapter contribute effectively. At the same guidance weight, our context-aware adapter considerably and consistently reduces the FID, resulting in
a significant improvement in image quality.
\vskip -0.1in
\paragraph{Training Convergence} We evaluate \method regarding our contribution to the model convergence. The comparison in \cref{fig:convergence} demonstrates that our context-aware adapter can significantly accelerate the training convergence and improve the semantic alignment between text and generated video. This observation also reveals the generalization ability of our contextualized diffusion. 

\begin{figure*}[ht!]
% \vskip -0.15in
 \begin{minipage}[t]{0.4\textwidth}
 \centering
  \captionsetup{width=.99\linewidth}
  %\hspace{0pt}
  \includegraphics[width=0.99\textwidth]{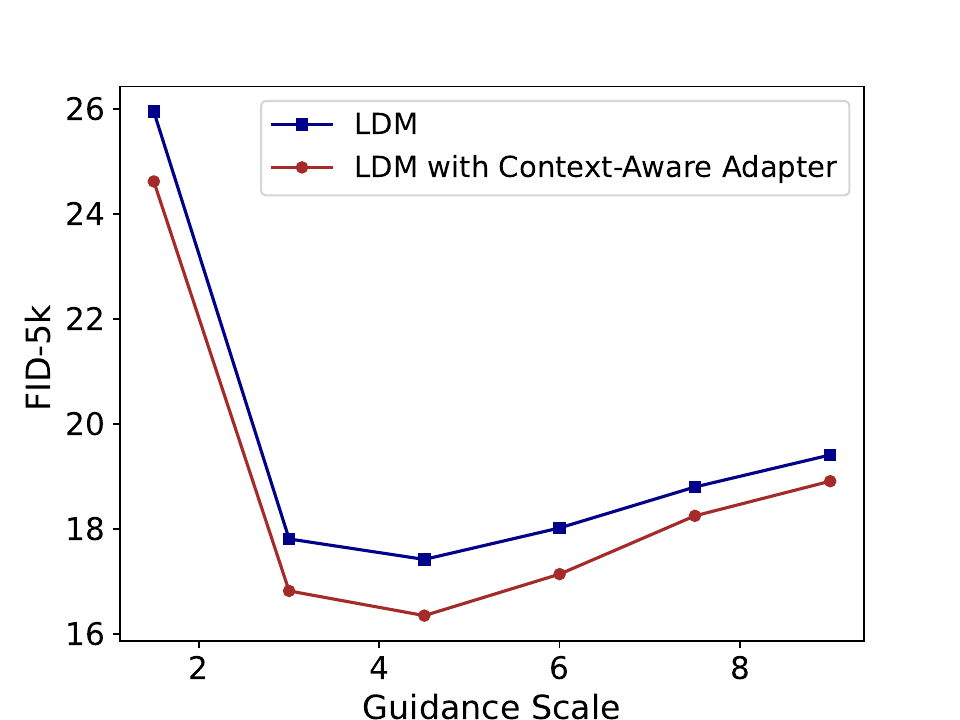}
  \caption{The trade-off between FID and CLIP scores for LDM and LDM with our context-aware adapter.}
  \label{fig:tradeoff}
 \end{minipage}
 \begin{minipage}[t]{0.6\textwidth}
 \centering
  \captionsetup{width=.75\linewidth}
  %\hspace{0pt}
  \includegraphics[width=8.5cm]{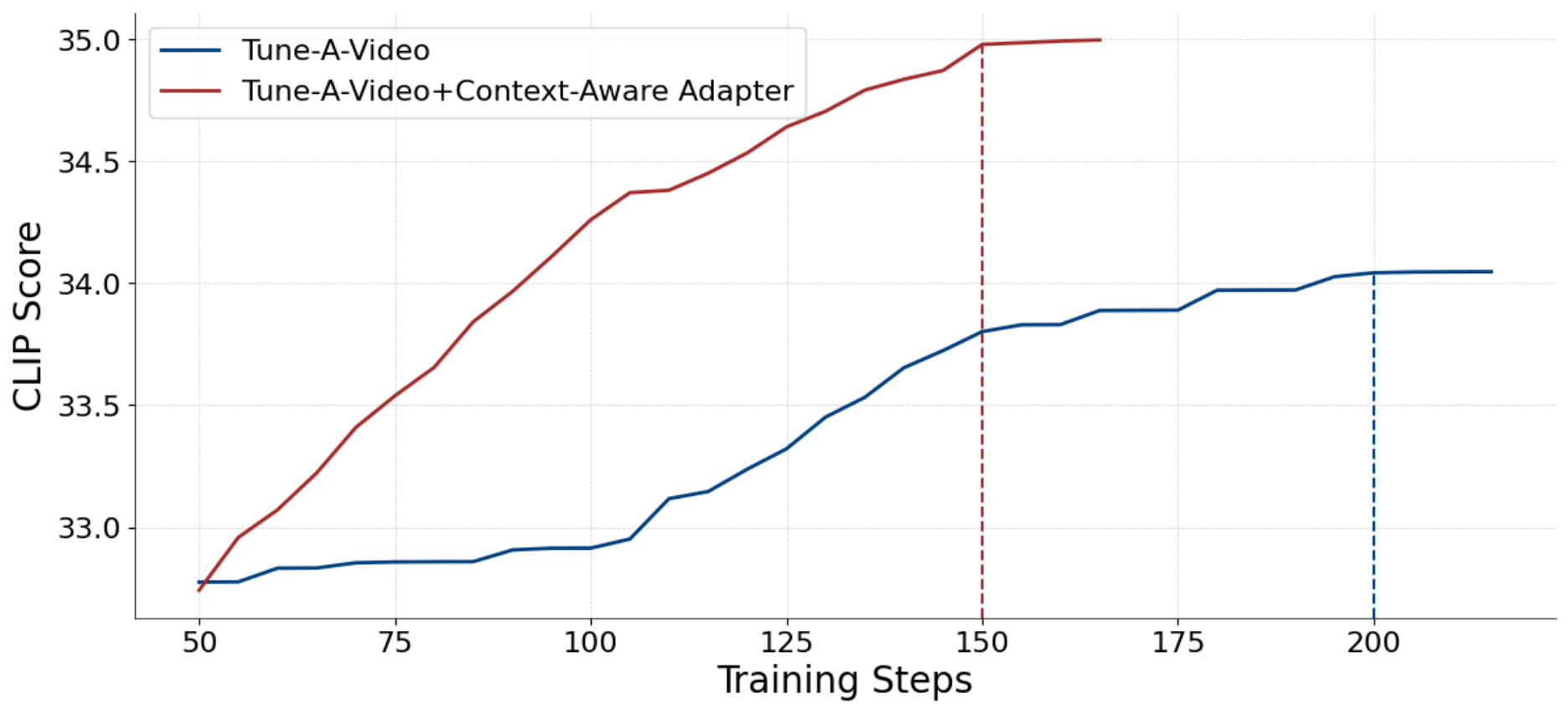}
  \caption{The comparison of model convergence between Tune-A-Video and Tune-A-Video + our context-aware adapter.}
\label{fig:convergence}
 \end{minipage}
\vskip -0.2in
\end{figure*}
\section{Conclusion}
In this paper, we propose a novel and general conditional diffusion model (\method) by propagating cross-modal context to all timesteps in both diffusion and reverse processes, and adapt their trajectories for facilitating the model capacity of cross-modal synthesis. 
We generalize our contextualized trajectory adapter to DDPMs and DDIMs with theoretical derivation, and consistently achieve state-of-the-art performance in two challenging tasks: text-to-image generation, and text-to-video editing. Extensive quantitative and qualitative results on the two tasks demonstrate the effectiveness and superiority of our proposed cross-modal contextualized diffusion models.

\bibliography{iclr2024_conference}
\bibliographystyle{iclr2024_conference}

\newpage
\appendix

\section{Theoretical Derivations}
\subsection{The Distributions in The Forward Process}\label{prof_process}
First, we derive the explicit expressions for $q(\vx_t|\vx_{t-1},\vx_0,\vc)$ and $q(\vx_{t-1}|\vx_t,\vx_0,\vc)$, based on our cross-modal contextualized diffusion defined by \cref{eq-ddpm-forward1}.
\begin{lemma}
    For the forward process $q(\vx_1,\vx_2,...,\vx_T|\vx_0,\vc) = \prod_{t=1}^T q(\vx_t|\vx_{t-1},\vx_0,\vc)$, if the transition kernel $q(\vx_t|\vx_{t-1},\vx_0,\vc)$ is defined as \cref{eq-ddpm-forward2}, then the conditional distribution $q(\vx_t|\vx_0,\vc)$ has the desired distribution as \cref{eq-ddpm-forward1}$, i.e., \mathcal{N}(\vx_t,\sqrt{\Bar{\alpha}_t}\vx_0+\textcolor{brown}{k_t\vr_\phi(\vx_0,\vc,t)},(1-\Bar{\alpha}_t)\mI)$.
\end{lemma}
\begin{proof}
    We prove the lemma by induction. Suppose at time t, we have $q(\vx_t|\vx_{t-1},\vx_0,\vc)$ and $q(\vx_{t-1}|\vx_0,\vc)$ admit the desired distributions as in \cref{eq-ddpm-forward1,eq-ddpm-forward2}, respectively, then we need to prove that $q(\vx_t|\vx_0,\vc)=\mathcal{N}(\vx_t,\sqrt{\Bar{\alpha}_t}\vx_0+\textcolor{brown}{k_t\vr_\phi(\vx_0,\vc,t)},(1-\Bar{\alpha}_t)\mI)$. We can re-write he conditional distributions of $\vx_t$ given $(\vx_{t-1},\vx_0,\vc)$ and $\vx_{t-1}$ given $(\vx_0,\vc)$ with the following equations: 
    \begin{equation}
        \begin{aligned}
                    \vx_t &= \sqrt{\alpha_t}\vx_{t-1}+\textcolor{brown}{k_t\vr_\phi(\vx_0,\vc,t)}-\sqrt{\alpha_{t}}\textcolor{brown}{k_{t-1}\vr_\phi(\vx_0,\vc,t-1)}+\sqrt{\beta_t}\epsilon_1,\label{eq-condition_dst_1}  \\
        \end{aligned}
    \end{equation}
    \begin{equation}                    
        \vx_{t-1} = \sqrt{\Bar{\alpha}_{t-1}}\vx_0+\textcolor{brown}{k_{t-1}\vr_\phi(\vx_0,\vc,t-1)}+\sqrt{1-\Bar{\alpha}_{t-1}}\epsilon_2 ,\label{eq-condition_dst_2} 
    \end{equation}
    where $\epsilon_1,\epsilon_2$ are two independent standard gaussian random variables. Replacing $\vx_{t-1}$ in \cref{eq-condition_dst_1} with \cref{eq-condition_dst_2}, we have:
    \begin{equation}
    \begin{aligned}
         \vx_t &= \sqrt{\Bar{\alpha}_{t}}\vx_0+\textcolor{brown}{k_{t}\vr_\phi(\vx_0,\vc,t)}\\
         &+\sqrt{\alpha_{t}}\textcolor{brown}{k_{t-1}\vr_\phi(\vx_0,\vc,t-1)} - \sqrt{\alpha_{t}}\textcolor{brown}{k_{t-1}\vr_\phi(\vx_0,\vc,t-1)}\\
         & + \sqrt{\beta_t}\epsilon_1+\sqrt{\alpha_t*(1-\Bar{\alpha}_{t-1})}*\epsilon_2\\
         & = \sqrt{\Bar{\alpha}_{t}}\vx_0+\textcolor{brown}{k_{t-1}\vr_\phi(\vx_0,\vc,t-1)}+ \sqrt{\beta_t}\epsilon_1+\sqrt{\alpha_t*(1-\Bar{\alpha}_{t-1})}*\epsilon_2\\
    \end{aligned}
    \end{equation}
    As a result, the distribution of $\vx_t$ given $(\vx_0,\vc)$ is a gaussian distribution with mean $\sqrt{\Bar{\alpha}_{t}}\vx_0+\textcolor{brown}{k_{t}\vr_\phi(\vx_0,\vc,t)}$ and  variance $\alpha_t*(1-\Bar{\alpha}_{t-1})+\beta_t = 1-\Bar{\alpha}_{t}$, which admits the desired distribution.
\end{proof}
\begin{proposition}
    Suppose the distribution of forward process is defined by \cref{eq-ddpm-forward1,eq-ddpm-forward2}, then at each time t, the posterior distribution $q(\vx_{t-1}|\vx_t,\vx_0,\vc)$ is described by \cref{eq-ddpm-backward}
\end{proposition}
\begin{proof}
    By the Bayes rule, $q(\vx_{t-1}|\vx_t,\vx_0,\vc)=\frac{q(\vx_{t-1}|\vx_0,\vc)q(\vx_t|\vx_{t-1},\vx_0,\vc)}{q(\vx_t|\vx_0,\vc)}$. By \cref{eq-ddpm-forward1,eq-ddpm-forward2}, the numerator and denominator are both gaussian , then the posterior distribution is also gaussian and we can proceed to calculate its mean and variance:
    \begin{equation}
        \begin{aligned}
            q(\vx_{t-1}|\vx_t,\vx_0,\vc) &= \frac{\mathcal{N}(\vx_{t-1},\sqrt{\Bar{\alpha}_{t-1}}\vx_0+\textcolor{brown}{\vb_{t-1}(\vx_0,\vc)},(1-\Bar{\alpha}_{t-1})\mI)}{\mathcal{N}(\vx_t,\sqrt{\Bar{\alpha}_t}\vx_0+\textcolor{brown}{\vb_t(\vx_0,\vc)},(1-\Bar{\alpha}_t)\mI)}\\
            &\quad *\mathcal{N}(\vx_t,\sqrt{\alpha_t}\vx_{t-1}+\textcolor{brown}{\vb_t(\vx_0,\vc)}-\sqrt{\alpha_{t}}\textcolor{brown}{\vb_{t-1}(\vx_0,\vc)},\beta_t\mI)\\
        \end{aligned},
    \end{equation}
    where  $\textcolor{brown}{\vb_t(\vx_0,\vc)}$ is an abbreviation form of $k_t\vr_\phi(\vx_0,\vc,t)$. Dropping the constants which are unrelated to $\vx_0,\vx_t,\vx_{t-1}$ and $\vc$, we have:
    \begin{equation}
        \begin{aligned}
            q(\vx_{t-1}|\vx_t,\vx_0,\vc) &\propto exp\left\{-\frac{(\vx_{t-1}-\sqrt{\Bar{\alpha}_{t-1}}\vx_0-\textcolor{brown}{\vb_{t-1}(\vx_0,\vc)})^2}{2(1-\Bar{\alpha}_{t-1})}+ \frac{(\vx_t-\sqrt{\Bar{\alpha}_t}\vx_0-
            \textcolor{brown}{\vb_t(\vx_0,\vc)})^2}{2(1-\Bar{\alpha}_t)}\right.\\
            &\quad - \left. \frac{(\vx_t-\sqrt{\alpha_t}\vx_{t-1}-\textcolor{brown}{\vb_t(\vx_0,\vc)}+\sqrt{\alpha_{t}}\textcolor{brown}{\vb_{t-1}(\vx_0,\vc)})^2}{2\beta_t}\right\}\\
            & = exp\left\{C(\vx_0,\vx_t,\vc)-\frac{1}{2}(\frac{1}{1-\Bar{\alpha}_{t-1}}+\frac{\alpha_t}{\beta_t})*\vx_{t-1}^2 +\vx_{t-1}*\right.\\
            &\left. [\frac{(\sqrt{\Bar{\alpha}_{t-1}}\vx_0+\textcolor{brown}{\vb_{t-1}(\vx_0,\vc)})}{1-\Bar{\alpha}_{t-1}}+\sqrt{\alpha_t}\frac{(\vx_t-\textcolor{brown}{\vb_t(\vx_0,\vc)}+\sqrt{\alpha_{t}}\textcolor{brown}{\vb_{t-1}(\vx_0,\vc)})}{\beta_t}]\right\}\\
            & = exp\left\{C(\vx_0,\vx_t,\vc)-\frac{1}{2}(\frac{1}{1-\Bar{\alpha}_{t-1}}+\frac{\alpha_t}{\beta_t})*\vx_{t-1}^2 +\vx_{t-1}*\right.\\
            &\left. [\frac{(\sqrt{\Bar{\alpha}_{t-1}}}{1-\Bar{\alpha}_{t-1}}\vx_0+\frac{\sqrt{\alpha_{t}}}{\beta_t}(\vx_t-\textcolor{brown}{\vb_t(\vx_0,\vc)})+(\frac{1}{1-\Bar{\alpha}_{t-1}}+\frac{\alpha_t}{\beta_t})*\textcolor{brown}{\vb_{t-1}(\vx_0,\vc)}]\right\},
        \end{aligned}
    \end{equation}
    where $C(\vx_0,\vx_t,\vc)$ is a constant term with respect to $\vx_{t-1}$. Note that $(\frac{1}{1-\Bar{\alpha}_{t-1}}+\frac{\alpha_t}{\beta_t}) = \frac{1-\Bar{\alpha}_{t}}{(1-\Bar{\alpha}_{t-1})(1-\alpha_t)}$, and with some algebraic derivation, we can show that the gaussian distribution $q(\vx_{t-1}|\vx_t,\vx_0,\vc)$ has:
    \begin{equation}
        \begin{aligned}
            variance: &\frac{(1-\Bar{\alpha}_{t-1})(1-\alpha_t)}{1-\Bar{\alpha}_{t}}\mI\\
            mean: &\frac{\sqrt{\Bar{\alpha}_{t-1}}\beta_t}{1-\Bar{\alpha}_t}\vx_0+\frac{\sqrt{\alpha_t}(1-\Bar{\alpha}_{t-1})}{1-\Bar{\alpha_t}}(\vx_t-\textcolor{brown}{\vb_t(\vx_0,\vc)})+\textcolor{brown}{\vb_{t-1}(\vx_0,\vc)}
        \end{aligned}
    \end{equation}
\end{proof}
Similarly, we can derive the distribution of DDIMs.
\begin{lemma}
    Suppose that at each time t, the posterior distribution is defined by a gaussin distribution with 
    \begin{equation}\label{ddim_posterior}
        \begin{aligned}
            &Mean: \sqrt{\bar{\alpha}_{t-1}}\vx_0+\sqrt{1-\bar{\alpha}_{t-1}-\sigma^2_t}*\frac{\vx_t-\sqrt{\bar{\alpha}_t}\vx_0}{\sqrt{1-\bar{\alpha}_t}} \\
            &\quad\quad -k_t\vr_\phi(\vx_0,\vc,t)*\frac{\sqrt{1-\bar{\alpha}_{t-1}-\sigma^2_t}}{\sqrt{1-\bar{\alpha}_t}}+k_{t-1}\vr_\phi(\vx_0,\vc,t-1)\\
            &Variance: \sigma_t^2\mI,
        \end{aligned}
    \end{equation}
    then the marginal distribution $q_\phi(\vx_t|\vx_0,\vc)$ has the desired distribution as \cref{eq-ddpm-forward1}
\end{lemma}
\begin{proof}
    We prove by induction. Suppose that at time t, posterior and marginal distributions admit the desired distributions, then we need to prove that at time $t-1$, $q_\phi(\vx_{t-1}|\vx_0,\vc)$ also has the desired distribution. Rewrite the posterior and marginal distribution as the following:
    \begin{equation}
        \begin{aligned}
            \vx_{t-1} = &\sqrt{\bar{\alpha}_{t-1}}\vx_0+\sqrt{1-\bar{\alpha}_{t-1}-\sigma^2_t}*\frac{\vx_t-\sqrt{\bar{\alpha}_t}\vx_0}{\sqrt{1-\bar{\alpha}_t}} \\&-k_t\vr_\phi(\vx_0,\vc,t)*\frac{\sqrt{1-\bar{\alpha}_{t-1}-\sigma^2_t}}{\sqrt{1-\bar{\alpha}_t}}+k_{t-1}\vr_\phi(\vx_0,\vc,t-1))+\sigma_t \epsilon_1\\
        \end{aligned}
    \end{equation}
    \begin{equation}                    
        \vx_{t} = \sqrt{\Bar{\alpha}_{t}}\vx_0+k_{t}\vr_\phi(\vx_0,\vc,t)+\sqrt{1-\Bar{\alpha}_{t}}\epsilon_2 , 
    \end{equation}
    where $\epsilon_1,\epsilon_2$ are standard gaussian noises. Plugging in $\vx_{t}$, we have:
    \begin{equation}
        \begin{aligned}
            \vx_{t-1} &= \sqrt{\bar{\alpha}_{t-1}}\vx_0\\
            &+k_t\vr_\phi(\vx_0,\vc,t)*\frac{\sqrt{1-\bar{\alpha}_{t-1}-\sigma^2_t}}{\sqrt{1-\bar{\alpha}_t}}-k_t\vr_\phi(\vx_0,\vc,t)*\frac{\sqrt{1-\bar{\alpha}_{t-1}-\sigma^2_t}}{\sqrt{1-\bar{\alpha}_t}}\\
            &+k_{t-1}\vr_\phi(\vx_0,\vc,t-1))+\sigma_t \epsilon_1+\sqrt{1-\bar{\alpha}_{t-1}-\sigma^2_t}\epsilon_2\\
            &=\sqrt{\bar{\alpha}_{t-1}}\vx_0+k_{t-1}\vr_\phi(\vx_0,\vc,t-1))+\sigma_t \epsilon_1+\sqrt{1-\bar{\alpha}_{t-1}-\sigma^2_t}\epsilon_2
        \end{aligned}
    \end{equation}
    Since the variance of $\sigma_t \epsilon_1+\sqrt{1-\bar{\alpha}_{t-1}-\sigma^2_t}\epsilon_2$ is $(1-\bar{\alpha}_{t-1})\mI$, we have the desired distribution.
\end{proof}
\subsection{Upper Bound of The Likelihood}\label{prof_objective}
Here we show with our parameterization, the objective function $\mathcal{L}_{\theta,\phi}$ \cref{eq_objective_final} is a upper bound of the negative log likelihood of the data distribution.
\begin{lemma}
    Based on the non-Markovian forward process  $q(\vx_1,\vx_2,...,\vx_T|\vx_0,\vc) = \prod_{t=1}^T q(\vx_t|\vx_{t-1},\vx_0,\vc)$ and the conditional reverse process $p_\theta(\vx_0,\vx_1,\vx_2,...,\vx_T|\vc) = p_\theta(\vx_T|\vc)\prod_{t=1}^T p_\theta(\vx_{t-1}|\vx_{t},\vc)$, the objective function \cref{eq-training-ori} is an upper bound of the negative log likelihood.\label{lemma_obj_1}
\end{lemma}
\begin{proof}
    \begin{equation}
        \begin{aligned}
            -\log p_\theta(\vx_0|\vc) 
            % & = -\log \E_{p_\theta(\vx_{1:T}|\vx_0,\vc)}\left\{\frac{q(\vx_{0:T}|\vc)}{p_\theta(\vx_{1:T}|\vx_0,\vc)}\right\}\\
            &\leq -\log p_\theta(\vx_0|\vc)+\E_{q(\vx_{1:T}|\vx_0,\vc)} \left\{-\log\frac{p_\theta(\vx_{1:T}|\vx_0,\vc)}{q(\vx_{1:T}|\vx_0,\vc)}\right\}\\
            &= \E_{q(\vx_{1:T}|\vx_0,\vc)}\left\{ -\log \frac{p_\theta(\vx_{0:T}|\vc)}{q(\vx_{1:T}|\vx_0,\vc)}\right\}\\
            &= -\E_{q(\vx_{1:T}|\vx_0,\vc)}\left\{ \log \frac{p_\theta(\vx_T|\vc)\prod_{t=1}^T p_\theta(\vx_{t-1}|\vx_{t},\vc)}{\prod_{t=1}^T q(\vx_t|\vx_{t-1},\vx_0,\vc)}\right\}\\
            &= -\E_{q(\vx_{1:T}|\vx_0,\vc)}\left\{ \log p_\theta(\vx_T|\vc)+\sum_{t>1}\log \frac{p_\theta(\vx_{t-1}|\vx_{t},\vc)}{q(\vx_t|\vx_{t-1},\vx_0,\vc)}+\log \frac{p_\theta(\vx_0|\vx_1,\vc)}{q(\vx_1|\vx_0,\vc)}\right\}\\
            &= -\E_{q(\vx_{1:T}|\vx_0,\vc)}\left\{\log p_\theta(\vx_T|\vc)+\log \frac{p_\theta(\vx_0|\vx_1,\vc)}{q(\vx_1|\vx_0,\vc)}\right.\\
            &\quad +\left. \sum_{t>1}\log \frac{p_\theta(\vx_{t-1}|\vx_{t},\vc)}{q(\vx_{t-1}|\vx_{t},\vx_0,\vc)}*\frac{q(\vx_{t-1}|\vx_0,\vc)}{q(\vx_{t}|\vx_0,\vc)}\right\}\\
            &= -\E_{q(\vx_{1:T}|\vx_0,\vc)}\left\{\log\frac{p_\theta(\vx_T|\vc)}{q(\vx_T|\vx_0,\vc)}+\log p_\theta(\vx_0|\vx_1,\vc)+\log\sum_{t>1}\frac{p_\theta(\vx_{t-1}|\vx_{t},\vc)}{q(\vx_{t-1}|\vx_{t},\vx_0,\vc)}\right\}\\
            &= \KL(q_\phi(\vx_T|\vx_0,\vc) \Vert p_{\theta}(\vx_T|\vc))-\E_{q(\vx_{1}|\vx_0,\vc)}\log p_{\theta}(\vx_0|\vx_1,\vc) \\
            &\quad +\sum_{t>1}\E_{q(\vx_t|\vx_0,\vc)} \KL(q_\phi(\vx_{t-1}|\vx_{t},\vx_0,\vc) \Vert p_{\theta}(\vx_{t-1}|\vx_t,\vc))
        \end{aligned}
    \end{equation}

\end{proof}
\begin{lemma}
    Assuming the relational network $\textcolor{brown}{\vr_\phi(\vx_0,\vc,t)}$ is Lipschitz continuous, i.e., $\forall t, \exists $a positive real number $ C_t$ s.t. $\Vert\vr_\phi(\vx_0,\vc,t)- \Vert\vr_\phi(\vx_0^{'},\vc,t)\Vert \leq C_t\Vert \vx_0-\vx_0^{'} \Vert$, then $\Vert \vf_\theta(\vx_t,\vc,t)-\vx_0\Vert^2_2$ is an upper bound of $\KL(q_\phi(\vx_{t-1}|\vx_{t},\vx_0,\vc) \Vert p_{\theta}(\vx_{t-1}|\vx_t,\vc))$ after scaling.\label{lemma_obj_2}
\end{lemma}
\begin{proof}
    From the main text, we know that 
    \begin{equation}\label{mean_l2}
        \begin{aligned}
                    & \KL(q_\phi(\vx_{t-1}|\vx_{t},\vx_0,\vc) \Vert p_{\theta}(\vx_{t-1}|\vx_t,\vc))
        \propto \\
        &\quad \big|\big|\vmu_\theta(\vx_t,c,t)-\frac{\sqrt{\Bar{\alpha}_{t-1}}\beta_t}{1-\Bar{\alpha}_t}\vx_0-\frac{\sqrt{\alpha_t}(1-\Bar{\alpha}_{t-1})}{1-\Bar{\alpha_t}}(\vx_t-\textcolor{brown}{\vb_t(\vx_0,\vc)})-\textcolor{brown}{\vb_{t-1}(\vx_0,\vc)}\big|\big|^2_{2},
        \end{aligned}
    \end{equation}
    where $\vmu_\theta(\vx_t,c,t)$ is the mean of $q_\theta(\vx_{t-1}|\vx_t,\vc)$. Here we discard a constant with respect to $\vx_0,\vx_t,\vc$. With our parameterization, 
    \begin{equation}
        \vmu_\theta(\vx_t,c,t) = \frac{\sqrt{\Bar{\alpha}_{t-1}}\beta_t}{1-\Bar{\alpha}_t}\hat{\vx}_0-\frac{\sqrt{\alpha_t}(1-\Bar{\alpha}_{t-1})}{1-\Bar{\alpha_t}}(\vx_t-\textcolor{brown}{\vb_t(\hat{\vx}_0,\vc)})-\textcolor{brown}{\vb_{t-1}(\hat{\vx}_0,\vc)},
    \end{equation}
    where $\hat{\vx}_0 = \vf_\theta(\vx_t,\vc,t)$. Thus the objective function can be simplified as:
    \begin{equation}
        \begin{aligned}
            &\big|\big|\frac{\sqrt{\Bar{\alpha}_{t-1}}\beta_t}{1-\Bar{\alpha}_t}(\hat{\vx}_0-\vx_0)+\frac{\sqrt{\alpha_t}(1-\Bar{\alpha}_{t-1})}{1-\Bar{\alpha_t}}(\textcolor{brown}{\vb_t(\hat{\vx}_0,\vc)}-\textcolor{brown}{\vb_t(\vx_0,\vc)})-(\textcolor{brown}{\vb_{t-1}(\hat{\vx}_0,\vc)}-\textcolor{brown}{\vb_{t-1}(\vx_0,\vc)})\big|\big|_2\\
            &\leq \frac{\sqrt{\Bar{\alpha}_{t-1}}\beta_t}{1-\Bar{\alpha}_t}\Vert\hat{\vx}_0-\vx_0\Vert_2+\frac{\sqrt{\alpha_t}(1-\Bar{\alpha}_{t-1})}{1-\Bar{\alpha_t}}\Vert\textcolor{brown}{\vb_t(\hat{\vx}_0,\vc)}-\textcolor{brown}{\vb_t(\vx_0,\vc)}\Vert_2+\Vert\textcolor{brown}{\vb_{t-1}(\hat{\vx}_0,\vc)}-\textcolor{brown}{\vb_{t-1}(\vx_0,\vc)}\Vert_2\\
            &\leq \frac{\sqrt{\Bar{\alpha}_{t-1}}\beta_t}{1-\Bar{\alpha}_t}\Vert\hat{\vx}_0-\vx_0\Vert_2+\frac{\sqrt{\alpha_t}(1-\Bar{\alpha}_{t-1})}{1-\Bar{\alpha_t}}k_tC_t\Vert\hat{\vx}_0-\vx_0\Vert_2+k_{t-1}C_{t-1}\Vert\hat{\vx}_0-\vx_0\Vert_2\\
            &= \lambda_t \Vert \vf_\theta(\vx_t,\vc,t)-\vx_0\Vert_2
        \end{aligned}
    \end{equation}
    Similar results can be proved for DDIMs by replacing the mean of posterior in DDPMs with DDIMs, defined by \cref{ddim_posterior}, in \cref{mean_l2}.
\end{proof}
Assume that the total diffusion step T is big enough and only a neglegible amount of noise is added to the data at the first diffusion step, then the term $\KL(q_\phi(\vx_T|\vx_0,\vc) \Vert p_{\theta}(\vx_T|\vc))-\E_{q(\vx_{1}|\vx_0,\vc)}\log p_{\theta}(\vx_0|\vx_1,\vc)$ is approximately zero. Now combining \cref{lemma_obj_1,lemma_obj_2}, we have the following proposition:
\begin{proposition}
    The objective function defined in \cref{eq_objective_final} is an upper bound of the negative log likelihood.
\end{proposition}

\subsection{Achieving Better Likelihood with \method}
Next, we show that \method is theoretically capable of achieving better likelihood compared to original DDPMs.  As the exact likelihood is intractable, we aim to compare the optimal variational bounds for negative log likelihoods. The objective function of \method at time step t is $E_{q_\phi}D_{KL}(q_\phi(\vx_{t-1}|\vx_t,\vx_0,\vc)||p_\theta(\vx_{t-1}|\vx_t,\vc))$, and its optimal solution is 
\begin{equation}
    \begin{aligned}
        &\min_{\phi,\theta}\E_{q_\phi}D_{KL}(q_\phi(\vx_{t-1}|\vx_t,\vx_0,\vc)||p_\theta(\vx_{t-1}|\vx_t,\vc))\\
        &=min_{\phi}[min_{\theta}\E_{q_\phi}D_{KL}(q_\phi(\vx_{t-1}|\vx_t,\vx_0,\vc)||p_\theta(\vx_{t-1}|\vx_t,\vc))]\\
        &\leq min_{\theta}\E_{q_{\phi=0}}D_{KL}(q_{\phi=0}(\vx_{t-1}|\vx_t,\vx_0,\vc)||p_\theta(\vx_{t-1}|\vx_t,\vc)),
    \end{aligned}
\end{equation}
where $\phi=0$ denotes setting the adapter network identical to 0, and thus $min_{\theta}\E_{q_{\phi=0}}D_{KL}(q_{\phi=0}(\vx_{t-1}|\vx_t,\vx_0,\vc)||p_\theta(\vx_{t-1}|\vx_t,\vc))$ is the optimal loss of origianl DDPMs objective at time t. Similar inequality can be obtained for t=1:
\begin{equation}
    \begin{aligned}
    &\min_{\phi,\theta}\E_{q_\phi}-\log p_\theta(\vx_0|\vx_1,\vc)\\
    &\leq  \min_{\theta}\E_{q_{\phi=0}}-\log p_\theta(\vx_0|\vx_1,\vc).
    \end{aligned}
\end{equation}
As a result,  we have the following inequality by summing up the objectives at all time step: 
\begin{equation}
    \begin{aligned} &-\E_{q(\vx_0)}\log p_\theta(\vx_0)\\
    &\leq min_{\phi,\theta} \sum_{t>1} \E_{q_\phi} D_{KL}(q_\phi(\vx_{t-1}|\vx_t,\vx_0,\vc)||p_\theta(\vx_{t-1}|\vx_t,\vc))+\E_{q_\phi}-\log p_\theta(\vx_0|\vx_1,\vc)+C\\
    &\leq min_{\theta} \sum_{t>1} \E_{q_{\phi=0}}D_{KL}(q_{\phi=0}(\vx_{t-1}|\vx_t,\vx_0,\vc)||p_\theta(\vx_{t-1}|\vx_t,\vc))+\E_{q_{\phi=0}}-\log p_\theta(\vx_0|\vx_1,\vc)+C
    \end{aligned}
\end{equation}
, where $C = \E\KL(q_\phi(\vx_T|\vx_0,\vc) \Vert p_{\theta}(\vx_T|\vc))$ is a constant defined by $\sqrt{\bar{\alpha}_T}$. Hence, \method has a tighter bound for the NLL, and thus theoretically capable of achieving better likelihood, compared with the original DDPMs.

\subsection{Better Expression of Cross-modal Semantics}
We provide an in-depth analysis on why \method can better express the cross-modal semantics. Our analysis focuses on the case of optimal estimation, as the theoretical analysis of convergence requires understanding the non-convex optimization of neural network, which is beyond the scope of this paper. Based on objective function in \cref{eq_objective_final}, the optimal solution of  \method at time t can be expressed as 
\begin{equation}
    \begin{aligned}
        &\arg\min_{\phi,\theta}\E_{q_\phi(\vx_t,\vx_0|\vc)}||\vx_0-f_\theta(\vx_t,\vc)||^2\\
        & = \arg\min_{\phi}\arg\min_{\theta}\E_{q_\phi(\vx_t,\vx_0|\vc)}||\vx_0-f_\theta(\vx_t,\vc)||^2\\
        &=\arg\min_{\phi}\E_{q_\phi(\vx_t|\vc)}E_{q_\phi(\vx_0|\vx_t,\vc)}||\vx_0-\E[\vx_0|\vx_t,\vc]||_2^2 \\
        & = \phi^*,\theta^*\\
    \end{aligned}
\end{equation}
since the best estimator under L2 loss is the conditional expectation. As a result, the optimal estimator of \method for $\vx_0$ is 
\begin{equation}
    \E[\vx_0|k_t \vr_{\phi^*}(\vx_0,c)+\sqrt{\bar{\alpha}_t}\vx_0+\sqrt{1-\bar{\alpha}_t}\mathbf{\epsilon}, \vc],
\end{equation}
while existing methods that did not incorporate cross-modal contextual information in the forward process have the following optimal estimator:
\begin{equation}
    \E[\vx_0|\sqrt{\bar{\alpha}_t}\vx_0+\sqrt{1-\bar{\alpha}_t}\mathbf{\epsilon}, \vc].
\end{equation}
Compared with existing methods,  \method can explicitly utilize the cross-modal context $\vr_{\phi^*}(x_0,c)$ to optimally recover the ground truth sample, and thus achieve better multimodal semantic coherence.

Furthermore, we analyze a toy example to show that \method can indeed utilize the cross-modal relations to better recover the ground truth sample. We consider the image embedding $x_0$ and text embedding $c$ that were generated with the following mechanism:
\begin{equation}
    x_0 = \mu(c)+\sigma(c)\epsilon,
\end{equation}
where $\epsilon$ is an independent standard gaussain, $\mu(c)$ and $\sigma^2(c)$ are the mean and variance of $x_0$ conditioned on $c$. We believe this simple model can capture the multimodal relationships in the embedding space, where the relevant images and text embeddings are closely aligned with each other. Then $x_t = \sqrt{\bar{\alpha}_t} x_0+\sqrt{1-\bar{\alpha}_t} \epsilon^{'}$ is the noisy image embedding in original diffusion model. We aim to calculate and compare the optimal estimation error at time step t in  original diffusion model and in \method:
\begin{equation}
    \begin{aligned}
    &\min_\theta \E||x_0-f_\theta(x_0,c)||_2^2 \\
 &=\E ||x_0-\E[x_0|x_t,c]||_2^2 \end{aligned}
\end{equation}
The conditional expectation as the optimal estimator of DDPMs can be calculated as:
$$\begin{aligned}\E[x_0|x_t,c]
&= \mu(c) - Cov(x_0,x_t|c)*Var(x_t|c)^{-1}(\sqrt{\bar{\alpha}_t}\mu(c)-x_t)\\
&= \mu(c) - \frac{\sqrt{\bar{\alpha}_t}\sigma(c)^2}{\bar{\alpha}_t\sigma(c)^2+1-\bar{\alpha}_t}(\sqrt{\bar{\alpha}_t}\mu(c)-x_t) \end{aligned}$$
As a result, we can calculate the estimation error of DDPMs:
\begin{equation}
    \begin{aligned} &\E||x_0-E[x_0|x_t,c]||_2^2 \\
&=E||\sigma(c)\epsilon-\frac{\sqrt{\bar{\alpha}_t}\sigma(c)^2}{\bar{\alpha}_t\sigma(c)^2+1-\bar{\alpha}_t}(\sqrt{\bar{\alpha}_t}\sigma(c)\epsilon+\sqrt{1-\bar{\alpha}_t} \epsilon^{'})||_2^2\\
& = d*\frac{(1- \bar{\alpha}_t)\bar{\alpha}_t\sigma(c)^4+ \sigma^2(c)(1-\bar{\alpha}_t)^2}{(\bar{\alpha}_t \sigma^2(c)+1- \bar{\alpha}_t)^2}\\
& = d *\sigma(c)^2 \frac{1- \bar{\alpha}_t}{\bar{\alpha}_t \sigma^2(c)+1- \bar{\alpha}_t}
 \end{aligned}\label{toy_original}
\end{equation}
Now we use \method with a parameterized adapter : $x_t = \sqrt{\bar{\alpha}_t} x_0+\sqrt{1-\bar{\alpha}_t} \epsilon^{'}+r(\phi,c,t)x_0$
 , where $r(\phi,c,t)x_0$ is the adapter. We can similarly calculate the conditional mean as the optimal estimator of \method:
 $$\E_\phi[x_0|x_t,c] = \mu(c)-\frac{\sigma^2(c)(r(\phi,c,t)+\sqrt{\bar{\alpha}_t})}{1-\bar{\alpha}_t+(r(\phi,c,t)+\sqrt{\bar{\alpha}_t})^2\sigma^2}*((r(\phi,c,t)+\sqrt{\bar{\alpha}_t})\mu(c)-x_t)$$
 And the estimation error for a given $\phi$ in \method is:
\begin{equation}
    \begin{aligned}
&\E||x_0-\E_\phi[x_0|x_t,c]||_2^2\\
&=\E||\sigma(c)\epsilon-\frac{\sigma^2(c)(r(\phi,c,t)+\sqrt{\bar{\alpha}_t})}{1-\bar{\alpha}_t+(r(\phi,c,t)+\sqrt{\bar{\alpha}_t})^2\sigma^2(c)}((r(\phi,c,t)+\sqrt{\bar{\alpha}_t})\sigma(c)\epsilon+\sqrt{1-\bar{\alpha}_t} \epsilon^{'})||_2^2\\
&= d \sigma(c)^2\frac{1-\bar{\alpha}_t}{1-\bar{\alpha}_t+(r(\phi,c,t)+\sqrt{\bar{\alpha}_t})^2\sigma^2(c)}
 \end{aligned}\label{toy_context}
\end{equation}
Comparing the denominators of two estimation errors (\cref{toy_original,toy_context}), we can see that using a non-negative adapter will always reduce the estimation error.

\section{More Model Analysis}
\paragraph{Comparison on Computational Costs}
We compare our method with LDM and Imagen regarding parameters, training time, and testing time in \cref{table:computational}. We find that our context-aware adapter (188M) only introduces few additional parameters and computational costs to the diffusion backbone (3000M), and substantially improves the generation performance, achieving a better trade-off than previous diffusion models.
\begin{table}[ht]
\caption{Comparison on Computational Costs.}
\centering
\begin{tabular}{l cccccc} 
\toprule
 Method &  \#Params & \#Training Cost & \# Inference Cost&FID \\ 
\midrule
LDM &1.4B& 0.39 s/Img&3.4 s/Img& 12.64\\
SDXL&10.3B&0.71 s/Img&9.7 s/Img&8.32\\
DALL·E 2&6.5B&2.38 s/Img&21.9 s/Img&10.39\\
Imagen&3B&0.79 s/Img&13.2 s/Img&7.27\\
Our \method&3B+188M&0.83 s/Img&13.4 s/Img&\textbf{6.48}\\
\bottomrule
\end{tabular}
\label{table:computational}
% \vspace{-0.2in}
\end{table}

\paragraph{Contributing to Faster Convergence}
In the \cref{fig:convergence} of main text, we generalize our context-aware adapter to other video diffusion model Tune-A-Video and make a faster and better model convergence. Here, we additionally generalize the adapter to image diffusion model LDM and plot the partial training curve on the subset of LAION dataset in \cref{app-convergence}. Similar to video domain, our context-aware adapter can also substantially improve the model convergence for image diffusion models, demonstrating the effectiveness and generalization ability of our method.
\begin{figure*}[ht]
% \vspace{-0.2in}
\centering
\includegraphics[width=0.7\textwidth]{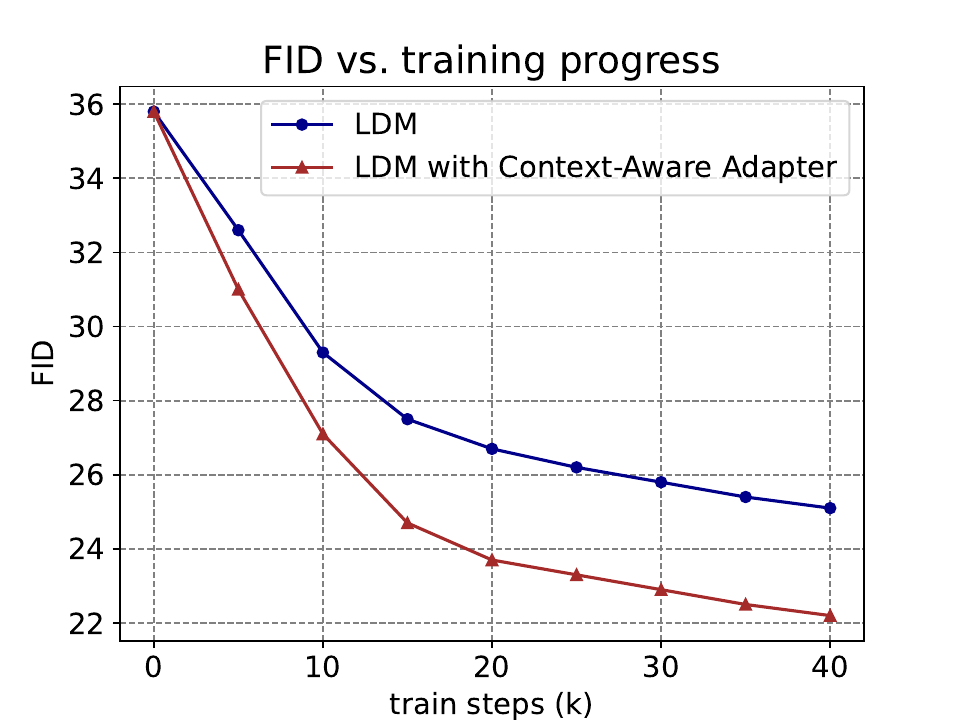}
\caption{Our context-aware adapter can enable faster model convergence (partial training curve)}.
\label{app-convergence}
\end{figure*}

\paragraph{Quantitative Results on Likelihood}
We follow \citep{kim2022maximum} to 
compute NLL/NELBO (Negaitve
Log-Likelihood/Negative Evidence Lower Bound) for performances of density estimation with Bits Per Dimension (BPD). We train our context-aware adapter on CIFAR-10 and compute NLL with the uniform dequantization. As the results in \cref{table-elbo}, we conclude that our method is empirically capable of achieving better likelihood compared to original DDPMs.
\begin{table}[ht]
\caption{NLL comparison on CIFAR-10.}
\centering
\begin{tabular}{l c} 
\toprule
 Method &  NLL $\downarrow$ \\ 
\midrule
DDPM \citep{jahn2021high}&3.01\\
DDPM + Context-Aware Adapter&\textbf{2.63}\\
\bottomrule
\end{tabular}
\label{table-elbo}
% \vspace{-0.2in}
\end{table}

\section{Hyper-parameters in \method} 
\label{app-hyperparameters}
We provide detailed hyper-parameters in training \method for text-to-image generation (in \cref{table:hyper-image}) and text-to-video editing (in \cref{table:hyper-video}).
% For one-shot text-to-video editing, we adopted the backbone diffusion models from official repositories, including Tune-A-Video \citep{wu2022tune}, FateZero \citep{qi2023fatezero}, and fine-tuned them with our Context-Aware Adapter. We used linear schedule to add noise, where $(\beta_{start},\beta_{end})=(0.00085,0.012)$. At training, the total diffusion time step is 1000, and we used DDIM with 20 time steps for inversion and sampling. We used  AdamW \citep{loshchilov2018decoupled} for optimization, and we set weight decay as 1$e$-2, learning rate as 1$e$-5, $\beta =(0.9, 0.999) $, and no warmup step. Here we provide the network detail in \cref{hyper-video}

\begin{table}[ht]
\small
\caption{Hyper-parameters in training our \method for text-to-image generation.}
\centering
\begin{tabular}{l c} 
\toprule
 Configs/Hyper-parameters &  Values \\ 
\midrule
 $T$ & 1000 \\ 
 Noise schedule & cosine \\
  Number of transformer blocks for cross-modal interactions & 4 \\
 Betas of AdamW \citep{loshchilov2018decoupled} & (0.9, 0.999) \\
 Weight decay & 0.0 \\
 Learning rate & 1$e$-4 \\
 Linear warmup steps & 10000 \\
Batch size &1024 \\

\bottomrule
\end{tabular}
\label{table:hyper-image}
% \vspace{-0.2in}
\end{table}
% \label{hyper}
% We give the hyper-parameters in Table \ref{table:hyper}. 
\begin{table}[ht]
\small
\caption{Hyper-parameters in training our \method for text-to-video editing.}
\centering
\begin{tabular}{l c} 
\toprule
 Configs/Hyper-parameters &  Values \\ 
\midrule
 $T$ & 20 \\ 
 Noise schedule & linear \\
      ($\beta_{start},\beta_{end}$) &$(0.00085,0.012)$\\
Number of transformer blocks for cross-modal interactions & 4 \\
Frames for causal attention &3\\
 Betas of AdamW \citep{loshchilov2018decoupled} & (0.9, 0.999) \\
 Weight decay & 1$e$-2 \\
 Learning rate & 1$e$-5 \\
 Warmup steps & 0 \\
Use checkpoint & True \\
Batch size &1 \\
Number of frames &8$\sim$24\\
Sampling rate &2 \\
\bottomrule
\end{tabular}
\label{table:hyper-video}
\vspace{-0.2in}
\end{table}

% \section{More Qualitative Comparisons} 

% \subsection{More Qualitative Comparisons on Text-to-Image Generation}
% \label{app-image-qualitative}
% In order to fully demonstrate the effectiveness of our proposed contextualized diffusion, we visualize more qualitative comparison results in \cref{app-fig-more-image}. The results sufficiently demonstrate the superior cross-modal understanding in generated images of our \method over other models.

% \begin{figure*}[ht]
% \centering
% \includegraphics[width=0.9\textwidth]{figs/piccase2.pdf}
% \vspace{-5mm}
% \caption{Synthesis examples demonstrating text-to-image capabilities of for various text prompts with LDM, Imagen, and
% \method (Ours). Our model can better express the semantics of the texts marked in blue. We use \textcolor{red}{red} boxes to highlight critical fine-grained parts where LDM and Imagen fail to align with texts. For example, in second row, only our method successfully generates the four letters spelling "LOVE". In third row, we generate the specific detail of a film roll, while other methods lose this detail.}
% \vspace{-8mm}
% \label{app-fig-more-image}
% \end{figure*}

\section{Generalizing to Other Text-Guided Video Diffusion Models}
\label{app-sec-general}

\paragraph{Qualitative Results} In order to fully demonstrate the generalization ability of the context-aware adapter in our contextualized diffusion, we visualize more qualitative comparison results, where we utilize context-aware adapter to improve Tune-A-Video \citep{wu2022tune} (in \cref{app-video-more-ablation-tune}) and FateZero \citep{qi2023fatezero} (in \cref{app-video-more-ablation-fatezero1} and \cref{app-video-more-ablation-fatezero2}). From the results, we observe that our context-aware adapter can effectiveness promote the performance of text-to-video editing, significantly enhancing the semantic alignment while maintaining structural information in source videos. All video examples are also provided in the supplementary material, and we are committed to open sourcing the train/inference code upon paper acceptance.

\begin{figure*}[ht]
\centering
\includegraphics[width=0.9\textwidth]{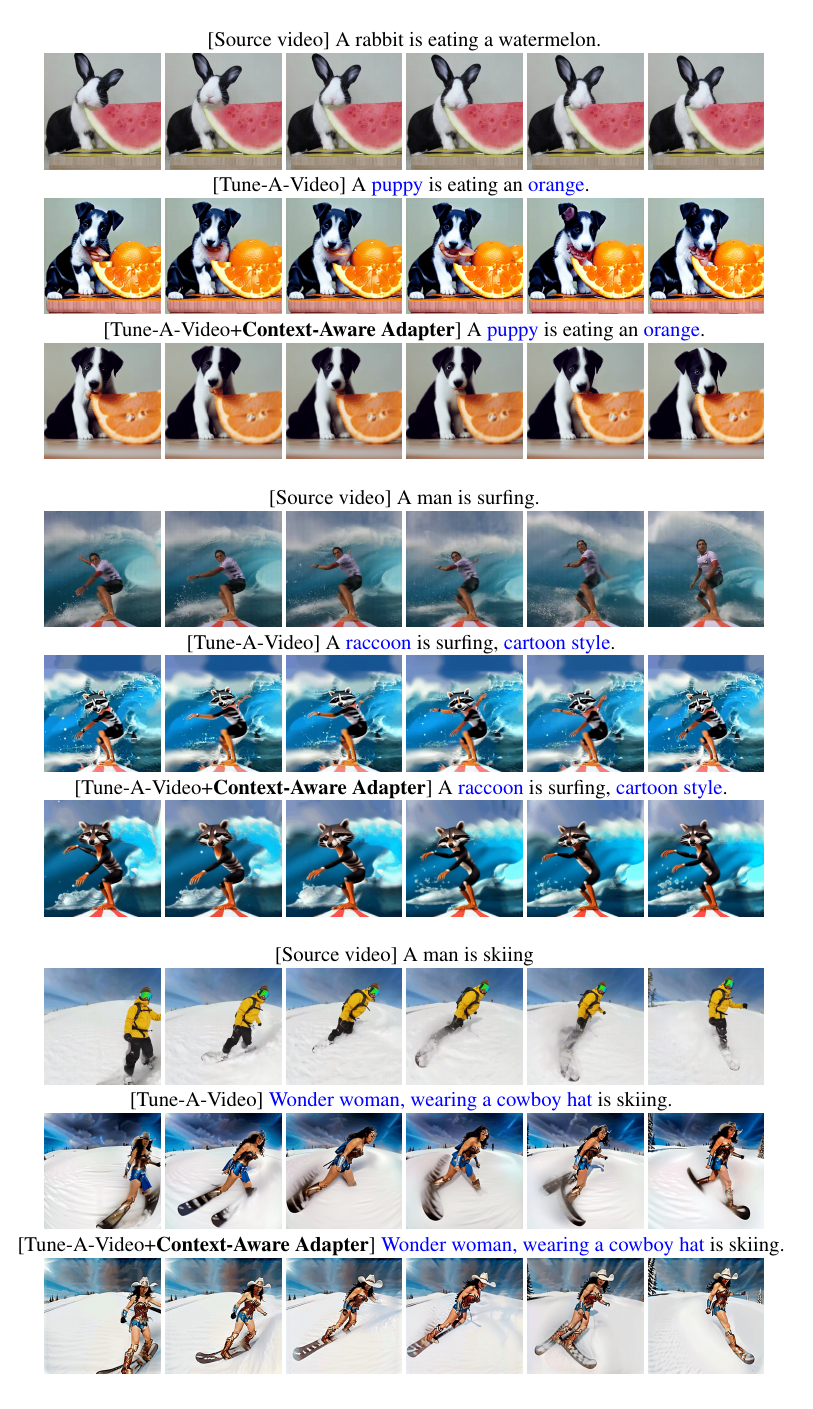}
\caption{Generalizing our context-aware adapter to Tune-A-Video \citep{wu2022tune}.}
\label{app-video-more-ablation-tune}
\end{figure*}

\begin{figure*}[ht]
\centering
\includegraphics[width=1.10\textwidth]{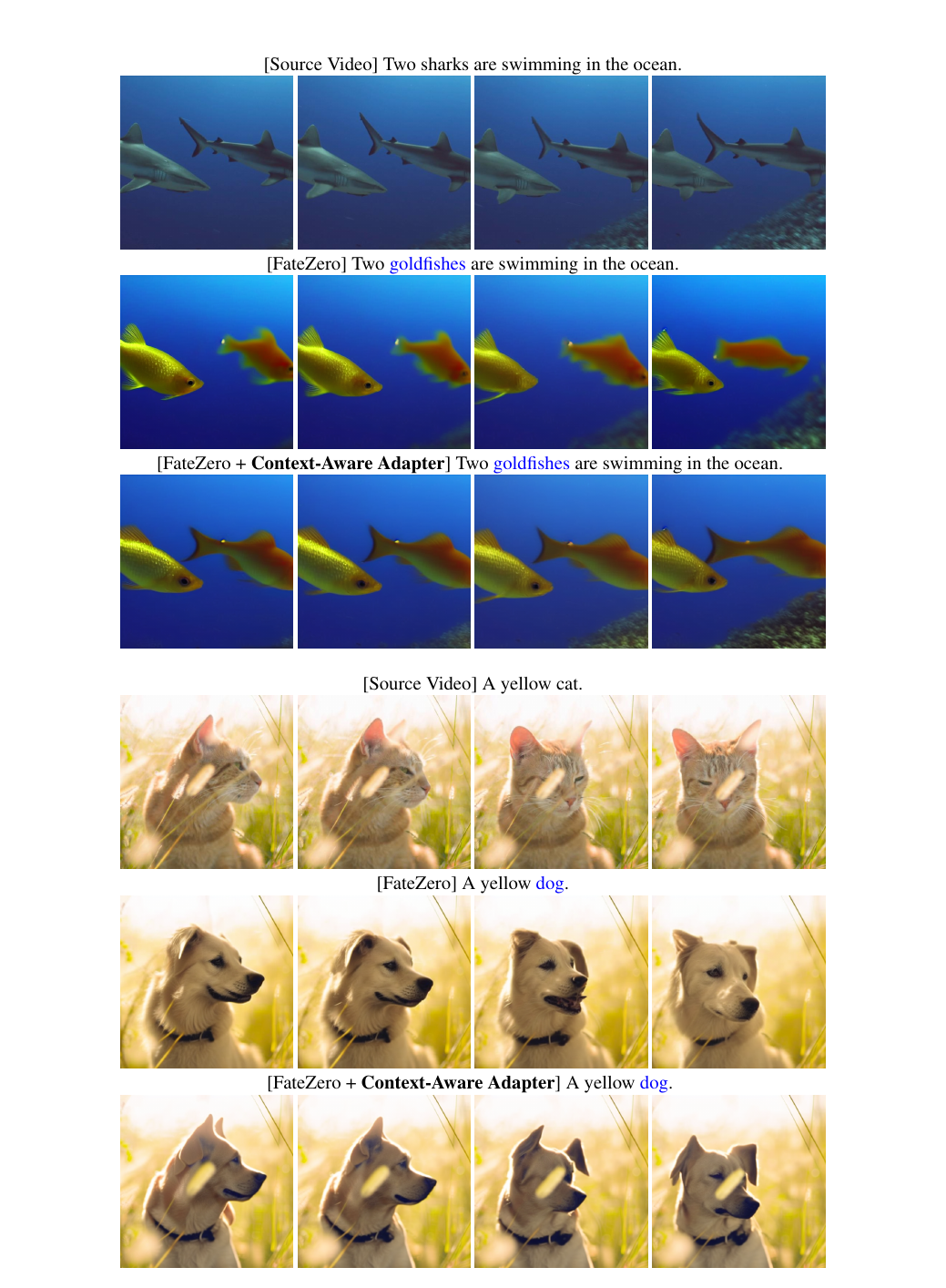}
\caption{Generalizing our context-aware adapter to FateZero \citep{qi2023fatezero}.}
\label{app-video-more-ablation-fatezero1}
\end{figure*}

\begin{figure*}[ht]
\centering
\includegraphics[width=1.10\textwidth]{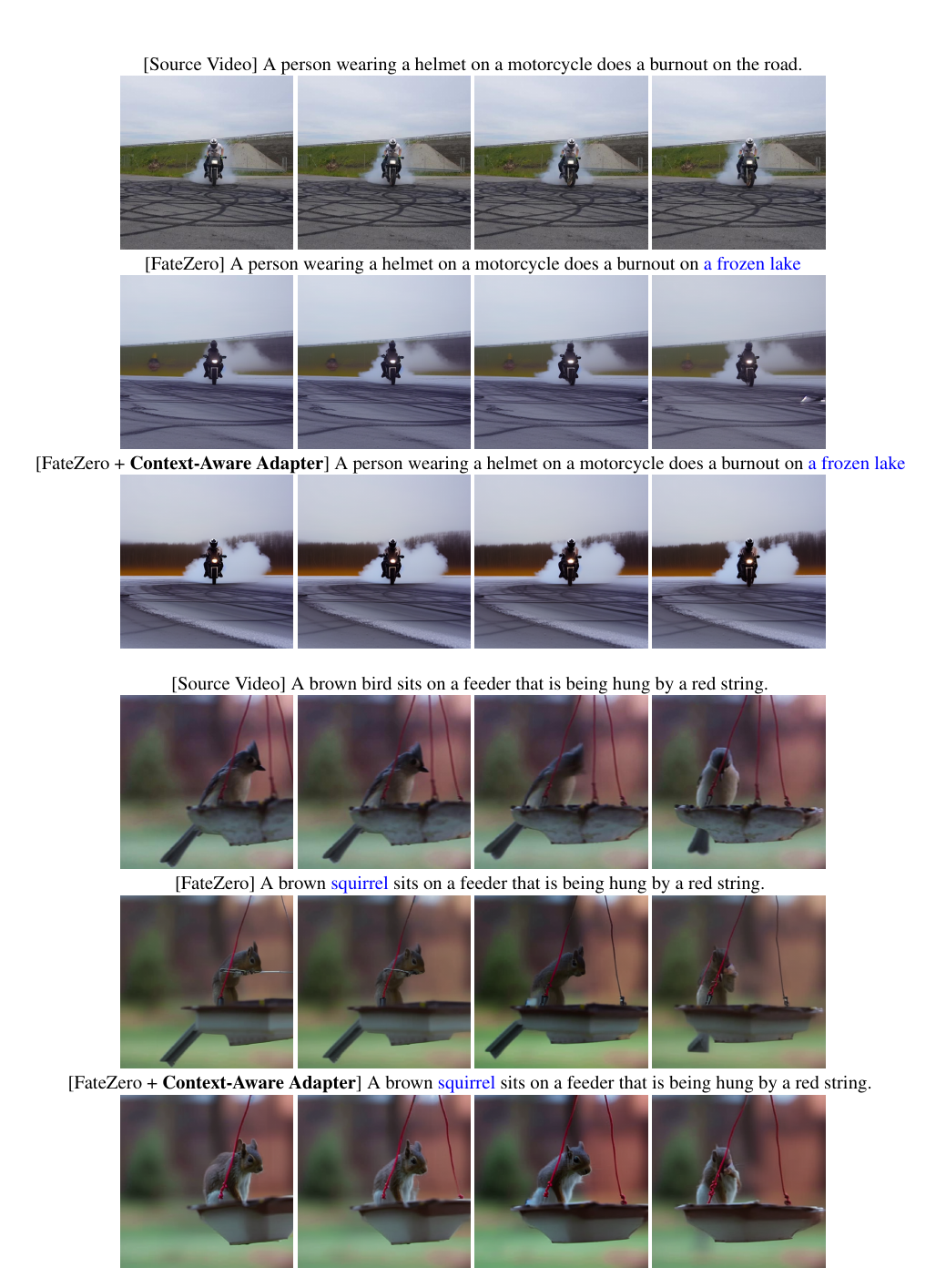}
\caption{Generalizing our context-aware adapter to FateZero \citep{qi2023fatezero}.}
\label{app-video-more-ablation-fatezero2}
\end{figure*}

\end{document}